\documentclass[journal, onecolumn]{IEEEtran}

\ifCLASSINFOpdf
\else
\fi

\usepackage{amsmath}
\usepackage{orcidlink}
\usepackage{multirow}
\usepackage{booktabs}
\usepackage{tabularx}
\usepackage{graphicx}
\usepackage{caption}
\usepackage{amsthm}
\usepackage{amsfonts}   
\usepackage{amssymb}    
\usepackage{colortbl,xcolor} 
\definecolor{lightred}{RGB}{255,230,230}
\definecolor{lightblue}{RGB}{230,240,255}
\usepackage[caption=false,font=footnotesize]{subfig}

\newtheorem{theorem}{Theorem}

\newtheorem{assumption}{Assumption}
\newtheorem{proposition}{Proposition}

\makeatletter
    \def\@cite#1#2{\textsuperscript{[{#1\if@tempswa , #2\fi}]}}
\makeatother

\begin{document}
\title{Autocorrelation Reintroduces Spectral Bias in KANs for Time Series Forecasting}

\author{
        Chen~Zeng,
        Jiahui~Wang,
    and~Qiao~Wang\orcidlink{0000-0002-5271-0472},~\IEEEmembership{Senior Member,~IEEE}%
 
\thanks{C. Zeng was with the School of Information Science and Engineering, Southeast University, Nanjing, China (email:
chenzeng@seu.edu.cn). }
\thanks{J. Wang was with the School of Economics and Management, Southeast University, Nanjing, China (email:
wangjh0521@seu.edu.cn). }
\thanks {Q. Wang was with both the School of Information Science and Engineering and the School of Economics and Management, Southeast University, Nanjing, China (Corresponding Author, email: qiaowang@seu.edu.cn).}}

\maketitle

\begin{abstract}
Existing theory suggests that Kolmogorov-Arnold Networks (KANs) can overcome the spectral bias commonly observed in neural networks under the assumption that inputs are statistically independent. However, this assumption does not hold in time series forecasting (TSF), where inputs are lagged observations with strong temporal autocorrelation. Through theoretical analysis and empirical validation, we obtain an unexpected finding: temporal autocorrelation reintroduces spectral bias in KANs, and the bias becomes increasingly pronounced as the degree of autocorrelation increases. This suggests that standard KANs may face substantial difficulties in TSF with strongly autocorrelated inputs. To address this problem, we introduce the Discrete Cosine Transform (DCT) to reduce the correlations among the network inputs. As expected, experimental results reveal that DCT preprocessing substantially reduces the observed low-frequency preference in TSF. This result also corroborates that the spectral bias of KANs in TSF tasks is indeed induced by the autocorrelation among input variables.
\end{abstract}


\IEEEpeerreviewmaketitle

\section{Introduction}

Time series forecasting (TSF) is a fundamental task across a wide range of domains, including financial markets\cite{financial}, weather prediction\cite{weather}, energy management\cite{energy}, and traffic flow estimation\cite{traffic}. In recent years, deep neural networks have become the dominant paradigm for TSF, with architectures such as multi-layer perceptrons (MLPs)\cite{mlp}, recurrent neural networks (RNNs)\cite{rnn, lstm}, Transformers\cite{transformer1, transformer2}, temporal convolutional networks (TCNs)\cite{tcn}, multi-scale temporal modeling frameworks like TimesNet\cite{timesnet}, and large language model-based approaches such as Time-LLM\cite{time-LLM} achieving strong empirical performance. Among the latest developments, Kolmogorov--Arnold Networks (KANs)\cite{kan} have attracted considerable attention due to their theoretical expressiveness rooted in the Kolmogorov--Arnold representation theorem\cite{kanthe1, kanthe2}, which states that any multivariate continuous function can be decomposed into a finite composition of univariate functions and addition. By replacing the fixed activation functions of conventional MLPs with learnable univariate spline functions on the edges, KANs offer a structurally different and potentially more expressive alternative for function approximation.

A well-documented phenomenon in neural network training is spectral bias\cite{mlp-spectral-bias1, mlp-spectral-bias2} (also referred to as the frequency principle), whereby networks tend to learn low-frequency components of the target function much faster than high-frequency components. This bias has been studied extensively in the context of MLPs and has been shown to degrade performance on tasks that require the faithful reconstruction of high-frequency patterns, such as image detail recovery, physics-informed modeling, and signal processing.

Recent theoretical analyses have provided encouraging results for KANs in this regard\cite{kan-spectral-bias}. Specifically, it has been shown that KANs, owing to their univariate spline parameterization, can in principle overcome the spectral bias that plagues conventional MLPs\cite{kan-noise, kan-regular}. The key insight is that B-spline basis functions, which serve as the elementary building blocks in KANs, possess localized support in the input space and can represent oscillatory components without the low-frequency preference inherent in the gradient dynamics of standard MLPs. These results suggest that KANs are well suited for learning functions with rich spectral content.

However, a critical assumption underlies virtually all existing analyses of spectral bias: the input variables are assumed to be statistically independent. Under this assumption, the eigenspectrum of the Neural Tangent Kernel (NTK) or the conjugate kernel exhibits a clean, separable structure that facilitates analysis\cite{ntk1, ntk2}. Most theoretical guarantees---including those claiming that KANs are free from spectral bias---are derived under this independence condition. In practice, this assumption is reasonable for many supervised learning tasks where input features are independently sampled or can be approximately decorrelated through preprocessing.

In TSF, however, this assumption fundamentally breaks down. The standard forecasting setup constructs input vectors from lagged observations of the same underlying stochastic process\cite{lag}, i.e., $\mathbf{X}_t = [x_{t-p+1}, x_{t-p+2}, \ldots, x_t]^\top$. By construction, the components of $\mathbf{X}_t$ are temporally autocorrelated, and this autocorrelation can be substantial---especially for smooth or slowly varying processes common in real-world applications\cite{autocorrelation}. Although this issue is of clear practical importance, existing studies have provided only limited empirical evidence\cite{bias-time-1, bias-time-2}, and rigorous theoretical analysis of how temporal autocorrelation among input variables affects the spectral bias of KANs or, more broadly, of neural networks in the TSF setting, remains largely absent.

In this paper, through a combination of theoretical analysis and systematic experiments, we arrive at an unexpected and practically significant finding: temporal autocorrelation among input variables reintroduces spectral bias in KANs, and the severity of this bias increases monotonically with the degree of autocorrelation. Concretely, by taking into account the correlation structure among input variables, we re-derive the condition number of the Hessian matrix of the mean squared error (MSE) loss with respect to the network weights during KAN training. We find that both the upper and lower bounds of this condition number are governed by the extreme eigenvalues of the input autocorrelation matrix~$\mathbf{R}$. Furthermore, as the correlation among input variables strengthens, $\mathbf{R}$ becomes increasingly ill-conditioned, causing both bounds of the Hessian condition number to grow. A larger Hessian condition number implies that the optimization landscape is more anisotropic across different frequency components, thereby exacerbating the spectral bias toward low frequencies. This prediction of the leading-order analysis is supported by controlled experiments on synthetic time series with varying autocorrelation strengths, where we observe that KANs progressively lose their ability to learn high-frequency components as autocorrelation increases.

To address this limitation, we propose a simple yet effective remedy: applying the Discrete Cosine Transform (DCT)\cite{dct} to the input vectors before feeding them into the KAN. The DCT is an orthogonal transform that approximately diagonalizes the covariance matrix of first-order Markov processes and, more generally, can reduce second-order temporal correlations in temporally correlated signals. By transforming autocorrelated lagged inputs into spectral coefficients with weaker cross-correlations, DCT makes the input representation closer to the weakly correlated setting under which KANs are expected to suffer less from spectral bias. Experimental results on synthetic TSF tasks demonstrate that DCT-KANs exhibit a substantially reduced low-frequency preference, supporting the effectiveness of the proposed approach. Furthermore, these results also provide empirical evidence for our theoretical analysis, suggesting that the spectral bias observed in KANs for TSF is largely associated with temporal autocorrelation among input variables rather than being solely an intrinsic limitation of the architecture itself.

The main contributions of this paper are summarized as follows:

\begin{itemize}
    \item We point out that existing studies on spectral bias are predominantly conducted under the assumption of statistically independent inputs, which does not hold in TSF where inputs are lagged observations with inherent temporal autocorrelation. This calls into question whether the guarantee of KANs being free from spectral bias remains valid in the TSF setting.
    
    \item Under a leading-order covariance-transfer approximation, we show that the upper and lower bounds of the effective Hessian condition number depend explicitly on the extreme eigenvalues of the input autocorrelation matrix. This analysis suggests that stronger temporal autocorrelation can amplify the anisotropy of the optimization landscape and thereby intensify spectral bias. Controlled experiments on synthetic time series further support this prediction.

    \item We propose DCT-KAN, which applies the DCT to decorrelate inputs before they enter the KAN. Experiments confirm that DCT-KAN substantially alleviates spectral bias in TSF, serving as converse evidence that the observed spectral bias is induced by input autocorrelation rather than being an intrinsic limitation of the KAN architecture.
\end{itemize}

\section{Background}

\subsection{KANs}

Kolmogorov--Arnold Networks (KANs)\cite{kan} are neural networks inspired by the Kolmogorov--Arnold representation theorem\cite{kanthe1, kanthe2} and are proposed as an alternative to conventional multi-layer perceptrons (MLPs). Unlike MLPs, which use fixed nonlinear activation functions on nodes and learn scalar weights on edges, KANs place learnable univariate functions on edges, while nodes only perform summation. Therefore, a KAN does not use conventional linear weight matrices; instead, each scalar weight is replaced by a trainable one-dimensional nonlinear function.

A general KAN can have arbitrary depth and width. Let its architecture be denoted by $[n_0,n_1,\ldots,n_L]$, where $n_0$ and $n_L$ are the input and output dimensions, respectively, and $n_1,\ldots,n_{L-1}$ are freely chosen hidden-layer widths. For the $l$-th layer, the edge connecting node $i$ in layer $l$ to node $j$ in layer $l+1$ is associated with a learnable function $\phi_{l,j,i}(\cdot)$. The activation of node $j$ in the next layer is computed as
\begin{equation}
x_{l+1,j}
=
\sum_{i=1}^{n_l}
\phi_{l,j,i}(x_{l,i}),
\qquad
j=1,\ldots,n_{l+1}.
\label{eq:kan_layer}
\end{equation}
Thus, each KAN layer can be regarded as a matrix of univariate functions, and the whole network is obtained by composing these function matrices:
\begin{equation}
\mathrm{KAN}(\mathbf{x})
=
(\Phi_{L-1}\circ \Phi_{L-2}\circ \cdots \circ \Phi_0)(\mathbf{x}).
\label{eq:kan_composition}
\end{equation}

In practice, each edge function is usually parameterized by a residual basis function and a spline component:
\begin{equation}
\phi(z)=b(z)+\operatorname{spline}(z),
\label{eq:kan_edge}
\end{equation}

where $b(z)$ is commonly chosen as the SiLU function and serves as a residual-like basis function in the edge activation. In other words, it provides a stable baseline transformation, while the spline term learns an additional data-adaptive correction. The spline component $\operatorname{spline}(z)$ is represented as a linear combination of B-spline basis functions with learnable coefficients. Owing to the local support of B-splines, KANs can flexibly refine nonlinear transformations on each edge.

Since the original KAN implementation based on B-splines can be computationally expensive, FastKAN\cite{fastkan} has been proposed as an efficient variant of KAN. FastKAN replaces the cubic B-spline basis functions in KANs with Gaussian radial basis functions (RBFs), which significantly simplifies the computation while preserving the functional form of learnable edge activations. Theoretically, FastKAN shows that KANs with spline bases can be viewed as RBF networks with fixed centers, establishing the functional equivalence between FastKAN and the original KAN formulation. Therefore, in the experimental part of this paper, we use FastKAN as the implementation of KAN for efficiency.

\subsection{Spectral Bias}

Spectral bias, also known as the frequency principle, refers to the training phenomenon that neural networks optimized by gradient-based methods tend to learn low-frequency components of the target function earlier than high-frequency components \cite{mlp-spectral-bias1, mlp-spectral-bias2}. To illustrate this idea, consider a target function $f$ defined on a bounded domain $\Omega$ and its Fourier decomposition
\begin{equation}
f(\mathbf{x})
=
\sum_{\boldsymbol{\omega}\in\mathcal{K}}
\hat{f}(\boldsymbol{\omega})
e^{\mathrm{i}\boldsymbol{\omega}^{\top}\mathbf{x}},
\label{eq:fourier_decomposition}
\end{equation}
where $\boldsymbol{\omega}$ denotes the frequency vector and $\hat{f}(\boldsymbol{\omega})$ is the corresponding Fourier coefficient. Let $f_{\boldsymbol{\theta}}$ be the function represented by a neural network during training. The frequency-wise approximation error can be written as
\begin{equation}
e_{\boldsymbol{\omega}}(t)
=
\left|
\hat{f}_{\boldsymbol{\theta}(t)}(\boldsymbol{\omega})
-
\hat{f}(\boldsymbol{\omega})
\right|.
\label{eq:frequency_error}
\end{equation}
A network is said to exhibit spectral bias if, for two frequencies satisfying
$\|\boldsymbol{\omega}_1\| < \|\boldsymbol{\omega}_2\|$, the low-frequency error
$e_{\boldsymbol{\omega}_1}(t)$ decays much faster than the high-frequency error
$e_{\boldsymbol{\omega}_2}(t)$ during training.

From an optimization perspective, spectral bias has been widely studied through the eigenspectrum of the Hessian matrix or the neural tangent kernel (NTK) associated with the loss landscape \cite{mlp-spectral-bias1, mlp-spectral-bias2, ntk1, ntk2}. For the mean squared error loss, in the linearized training regime, the parameter-space Hessian and the function-space NTK share the same nonzero eigenvalues. Therefore, the convergence speed of each functional component is governed by the eigenvalue associated with the corresponding eigenfunction.

Existing studies\cite{mlp-spectral-bias1, mlp-spectral-bias2} have established a close correspondence between these eigenvalues and the frequency content of the learned function. In particular, when the input distribution is uniform on a periodic or compact domain, or when the kernel is approximately translation-invariant, Fourier modes serve as exact or approximate eigenfunctions of the NTK or related kernel operators. In this case, each frequency component can be associated with a kernel eigenvalue, denoted by $\lambda(\boldsymbol{\omega})$. For conventional MLPs, these eigenvalues typically decay as the frequency magnitude $\|\boldsymbol{\omega}\|$ increases. Consequently, low-frequency Fourier modes correspond to larger eigenvalues and are learned faster, whereas high-frequency modes correspond to smaller eigenvalues and are learned more slowly.

More generally, even when the kernel is not exactly diagonalized by the Fourier basis, the dominant eigenfunctions of neural-network kernels are usually smoother and less oscillatory, while highly oscillatory components are associated with smaller eigenvalues. This provides an eigenvalue-based explanation for the frequency principle: gradient descent progresses rapidly along large-eigenvalue directions, which are typically aligned with low-frequency components, and slowly along small-eigenvalue directions, which are typically aligned with high-frequency components.

In a linearized regime, if the error is decomposed along kernel eigenfunctions as
\begin{equation}
e(t)
=
\sum_{m} a_m(t)\psi_m,
\end{equation}
then the gradient flow dynamics satisfy
\begin{equation}
a_m(t)
=
a_m(0)\exp(-\lambda_m t),
\label{eq:kernel_decay}
\end{equation}
where $\lambda_m$ is the eigenvalue corresponding to eigenfunction $\psi_m$. Therefore, under the Hessian/NTK-based characterization of spectral bias, 
the severity of spectral bias can be quantified by the disparity between 
the eigenvalues associated with low- and high-frequency components. 
A more ill-conditioned Hessian or NTK spectrum implies a larger gap between 
the convergence rates of different frequency modes, and hence a stronger 
frequency-dependent learning bias \cite{mlp-spectral-bias1, mlp-spectral-bias2}.

For conventional MLPs, especially ReLU or tanh networks, spectral bias has been widely observed and theoretically analyzed \cite{mlp-spectral-bias1, mlp-spectral-bias2}. These studies show that, from the Hessian/NTK perspective, low-frequency components are usually aligned with better-conditioned directions of the optimization problem, whereas high-frequency components correspond to poorly conditioned directions. Consequently, a larger effective condition number indicates a larger separation between the learning speeds of low- and high-frequency components, and thus a more severe spectral bias. For a two-layer ReLU MLP with width $n$, the condition number of the Hessian matrix $H_{\mathrm{MLP}}$ associated with the least-squares objective has been shown to scale as
\begin{equation}
\kappa(H_{\mathrm{MLP}})
=
\frac{\lambda_{\max}(H_{\mathrm{MLP}})}
{\lambda_{\min}^{+}(H_{\mathrm{MLP}})}
\propto n^{4},
\label{eq:mlp_condition_number}
\end{equation}
where $\lambda_{\min}^{+}(\cdot)$ denotes the smallest positive eigenvalue. This large condition number indicates that the optimization landscape of MLPs is highly anisotropic. Consequently, gradient descent makes rapid progress along some dominant directions, typically corresponding to low-frequency components, while progress along high-frequency directions is much slower. This provides an optimization-based explanation for the strong low-frequency preference of MLPs.

Compared with MLPs, KANs exhibit substantially different spectral behavior because their learnable transformations are placed on edges and parameterized by local univariate basis functions. Existing analysis mainly considers a single-layer KAN without the SiLU residual term\cite{kan-spectral-bias}. Given an input $\mathbf{x}\in\mathbb{R}^{d}$ and an output dimension $d'$, such a KAN can be written as
\begin{equation}
\mathrm{KAN}(\mathbf{x},\boldsymbol{\theta})_i
=
\sum_{j=1}^{d}
\sum_{\ell=1}^{G+k-1}
c_{ij\ell} B_{\ell}(x_j),
\qquad
i=1,\ldots,d',
\label{eq:shallow_kan}
\end{equation}
where $B_{\ell}(\cdot)$ denotes the $\ell$-th B-spline basis function, $G$ is the grid size, $k$ is the spline degree, and $\boldsymbol{\theta}=\{c_{ij\ell}\}$ represents the trainable spline coefficients. Since the model is linear with respect to the coefficients $c_{ij\ell}$, the continuous least-squares loss
\begin{equation}
\mathcal{L}(\boldsymbol{\theta})
= \frac{1}{2}
\int_{\Omega}
\left\|
f(\mathbf{x})
-
\mathrm{KAN}(\mathbf{x},\boldsymbol{\theta})
\right\|^2
d\mathbf{x},
\qquad
\Omega=[-1,1]^d,
\label{eq:kan_l2_loss}
\end{equation}
is a quadratic function of $\boldsymbol{\theta}$:
\begin{equation}
\mathcal{L}(\boldsymbol{\theta})
=
\frac{1}{2}\boldsymbol{\theta}^{\top}M\boldsymbol{\theta}
+
\mathbf{b}^{\top}\boldsymbol{\theta}
+
C,
\label{eq:kan_quadratic_loss}
\end{equation}
where $M$ is the Hessian matrix of the loss.

The entries of the Hessian matrix $M$ are determined by the inner products between B-spline basis functions:
\begin{equation}
M_{(i,j,\ell),(i',j',\ell')}
=
\begin{cases}
\displaystyle
\int_{\Omega}
B_{\ell}(x_j)B_{\ell'}(x_{j'})
d\mathbf{x},
& i=i',\\[2mm]
0,
& i\neq i'.
\end{cases}
\label{eq:kan_hessian}
\end{equation}
Therefore, the convergence of gradient descent is governed by the eigenspectrum of $M$. Let
\[
0\leq \lambda_1(M)\leq \lambda_2(M)\leq \cdots \leq \lambda_N(M)
\]
be the eigenvalues of $M$, where
\begin{equation}
N=(G+k-1)dd'.
\end{equation}
For a single-layer KAN, the following effective condition number is bounded:
\begin{equation}
\kappa_{\mathrm{KAN}}
=
\frac{\lambda_N(M)}
{\lambda_{d'(d-1)+1}(M)}
\leq C_k d,
\label{eq:kan_condition_number}
\end{equation}
where $C_k$ is a constant depending only on the spline degree $k$. The first $d'(d-1)$ degenerate eigen-directions are excluded because the KAN parameterization is not unique: for example, constant functions can be represented through B-spline expansions along different input coordinates. After removing these degenerate directions, the remaining Hessian spectrum is well conditioned.

Equation~\eqref{eq:kan_condition_number} provides the key theoretical explanation for why KANs are expected to suffer much less from spectral bias than MLPs. Unlike the MLP condition number in \eqref{eq:mlp_condition_number}, which grows rapidly with the network width, the effective condition number of a shallow KAN is bounded by a quantity that is independent of the grid size $G$. This means that increasing the number of spline grid points improves the representational resolution of KANs without necessarily making the optimization problem increasingly ill-conditioned. Consequently, different frequency components tend to be learned at more comparable rates.

The local support of B-spline basis functions further explains this behavior. Each B-spline basis function is active only on a small interval of the input domain, allowing KANs to model localized and oscillatory structures more directly than MLPs with global activation patterns. In this sense, the grid size $G$ controls the finest frequency scale that a single KAN layer can represent. A larger $G$ provides a finer partition of the input domain and hence improves the ability of the KAN to represent high-frequency components. Moreover, the grid extension strategy used in KAN training can be viewed as a multi-resolution learning mechanism: the network can first learn coarse structures on a small grid and then progressively refine the spline grid to capture finer details.

Nevertheless, It should be noted that the above conclusions are derived under the assumption that different input variables are statistically independent\cite{mlp-spectral-bias1, kan-spectral-bias}. This assumption is essential for the existing spectral-bias analysis of KANs, but it does not generally hold in time series forecasting, where input variables are constructed from lagged observations.

\section{Spectral Bias of KANs in Time Series Forecasting}

The preceding discussion shows that KANs are expected to be much less affected by spectral bias than MLPs when the input variables are statistically independent. However, this independence assumption is generally violated in time series forecasting. In a standard autoregressive forecasting setting, the input vector is constructed from consecutive lagged observations of the same time series, and hence different input coordinates are naturally autocorrelated. In this section, we show that such temporal autocorrelation changes the Hessian spectrum of KAN training and reintroduces spectral bias.

\subsection{Problem Setup and Notation}

Consider a univariate time series $\{x_t\}_{t\in\mathbb{Z}}$. Given a look-back window of length $p$, the input vector is defined as
\begin{equation}
\mathbf{X}_t
=
[X_{t,1},X_{t,2},\ldots,X_{t,p}]^\top
=
[x_t,x_{t-1},\ldots,x_{t-p+1}]^\top .
\end{equation}
We assume that each input coordinate has been normalized. Let
\begin{equation}
R_{jj'}
=
\operatorname{Corr}(X_{t,j},X_{t,j'})
=
r(|j-j'|)
\end{equation}
denote the autocorrelation matrix of the lagged input vector. For a stationary time series, $R$ is a symmetric Toeplitz correlation matrix with $\operatorname{Tr}(R)=p$.

Following the standard spectral-bias analysis of KANs, we focus on a single-layer KAN and omit the residual SiLU term. Let
\begin{equation}
m = G+k-1
\end{equation}
be the number of B-spline basis functions, where $G$ is the grid size and $k$ is the spline degree. The model is written as
\begin{equation}
\hat{y}_{t+1}
=
f_{\boldsymbol{\theta}}(\mathbf{X}_t)
=
\sum_{j=1}^{p}
\sum_{\ell=1}^{m}
c_{j\ell}B_{\ell}(X_{t,j}),
\label{eq:tsf_kan_model}
\end{equation}
where $B_{\ell}(\cdot)$ denotes the $\ell$-th B-spline basis function. The population MSE loss is
\begin{equation}
\mathcal{L}(\boldsymbol{\theta})
=
\frac{1}{2}
\mathbb{E}
\left[
\left(
y_{t+1}
-
f_{\boldsymbol{\theta}}(\mathbf{X}_t)
\right)^2
\right].
\label{eq:tsf_kan_loss}
\end{equation}

\subsection{Leading-Order Hessian Approximation Under Temporal Autocorrelation}

Before deriving the Hessian structure, we emphasize that the following analysis is based on an analytical closure approximation. For nonlinear basis functions, the cross-covariance matrix 
\(\operatorname{Cov}(\mathbf{b}(X_{t,j}),\mathbf{b}(X_{t,j'}))\) is generally not determined solely by the scalar correlation coefficient \(R_{jj'}\). It also depends on the full joint distribution of the lagged variables and on the nonlinear shape of the basis functions. Therefore, the Kronecker-product Hessian derived below should be interpreted as a leading-order approximation rather than as an exact identity for arbitrary autocorrelated inputs.

Define the B-spline basis vector
\begin{equation}
\mathbf{b}(z)
=
[B_1(z),B_2(z),\ldots,B_m(z)]^\top .
\end{equation}
Let
\begin{equation}
\boldsymbol{\nu}
=
\mathbb{E}[\mathbf{b}(z)],
\qquad
C
=
\mathbb{E}[\mathbf{b}(z)\mathbf{b}(z)^\top],
\qquad
D
=
\boldsymbol{\nu}\boldsymbol{\nu}^\top .
\end{equation}
Here, $C$ is the Gram matrix of the B-spline basis functions, while $C-D$ is the covariance matrix of the centered basis vector.

\begin{assumption}[Covariance-transfer closure]
\label{ass:cov_transfer}
For analytical tractability, we approximate the cross-covariance matrix of the transformed lagged inputs by transferring the scalar autocorrelation coefficient of the original lagged variables to the covariance of the transformed basis vector:
\begin{equation}
\operatorname{Cov}
\left(
\mathbf{b}(X_{t,j}),
\mathbf{b}(X_{t,j'})
\right)
\approx
R_{jj'}(C-D).
\label{eq:cov_transfer_assumption}
\end{equation}
This approximation is exact for the diagonal blocks with \(j=j'\), and is also consistent with the independent-input case where \(R_{jj'}=0\) for \(j\neq j'\). For general nonlinear basis functions and general joint distributions, however, \eqref{eq:cov_transfer_assumption} should be understood as a leading-order closure approximation rather than an exact identity.
\end{assumption}

\begin{proposition}[Leading-order Hessian under autocorrelated lagged inputs]
\label{prop:tsf_hessian}
Under Assumption~\ref{ass:cov_transfer}, the population Hessian 
\(M=\nabla^2_{\boldsymbol{\theta}}\mathcal{L}(\boldsymbol{\theta})\) of the single-layer KAN admits the following leading-order block approximation:
\begin{equation}
M_{jj'}
\approx
D+R_{jj'}(C-D),
\qquad
j,j'=1,\ldots,p,
\label{eq:tsf_hessian_block}
\end{equation}
where each block \(M_{jj'}\in\mathbb{R}^{m\times m}\). Equivalently,
\begin{equation}
M
\approx
J_p\otimes D
+
R\otimes(C-D),
\label{eq:tsf_hessian_kron}
\end{equation}
where \(J_p=\mathbf{1}_p\mathbf{1}_p^\top\) and \(\otimes\) denotes the Kronecker product.
\end{proposition}

\noindent\emph{Proof.} The detailed derivation is provided in Appendix~\ref{app:derivation_tsf_hessian}. \hfill $\square$

Proposition~\ref{prop:tsf_hessian} provides a leading-order generalization of the Hessian structure used in the original KAN spectral-bias analysis. If the input coordinates are independent, then \(R=I_p\), and \eqref{eq:tsf_hessian_kron} becomes
\begin{equation}
M
\approx
J_p\otimes D
+
I_p\otimes(C-D).
\end{equation}
In this case, the diagonal blocks of \(M\) are \(C\), while the off-diagonal blocks are \(D\), which recovers the block structure used in the standard analysis of KANs~\cite{kan-spectral-bias}. Therefore, the independent-input result is recovered as a special case of the leading-order approximation.

In time series forecasting, however, $R\neq I_p$ in general. The additional factor $R_{jj'}$ in \eqref{eq:tsf_hessian_block} means that the Hessian spectrum is no longer determined only by the B-spline Gram matrix $C$, but is also shaped by the temporal autocorrelation matrix $R$.

\subsection{Effective Condition Number}

As in the standard KAN analysis, the parameterization contains degenerate directions due to the non-uniqueness of decomposing constant functions across different input coordinates. For the single-output case with $p$ input lags, there are $p-1$ such degenerate directions. We therefore define the effective condition number as
\begin{equation}
\kappa_{\mathrm{TSF}}(M)
=
\frac{\lambda_{\max}(M)}
{\lambda_p(M)},
\end{equation}
where the eigenvalues are arranged in nondecreasing order and $\lambda_p(M)$ denotes the first non-degenerate eigenvalue.

\begin{theorem}[Bounds under the leading-order Hessian approximation]
\label{thm:tsf_condition}
Let \(R\) be positive definite and let \(C\) be the B-spline Gram matrix. Under the leading-order Hessian approximation in Proposition~\ref{prop:tsf_hessian}, the following bounds hold:

\begin{equation}
\frac{1}{2}\lambda_{\max}(R)\lambda_{\max}(C)
\leq
\lambda_{\max}(M)
\leq
p\lambda_{\max}(C),
\label{eq:lambda_max_M_bound}
\end{equation}
and
\begin{equation}
\lambda_{\min}(R)\lambda_{\min}(C)
\leq
\lambda_p(M)
\leq
\lambda_{\min}(R)\lambda_{\max}(C).
\label{eq:lambda_p_M_bound}
\end{equation}
Consequently,
\begin{equation}
\frac{\lambda_{\max}(R)}
{2\lambda_{\min}(R)}
\leq
\kappa_{\mathrm{TSF}}(M)
\leq
\frac{p}
{\lambda_{\min}(R)}
C_k,
\label{eq:tsf_condition_bound}
\end{equation}
where
\begin{equation}
C_k
=
\frac{\lambda_{\max}(C)}
{\lambda_{\min}(C)}
\end{equation}
is the condition number associated with the B-spline Gram matrix\cite{kan-spectral-bias}.
\end{theorem}

\noindent\emph{Proof.} The proof is given in Appendices~\ref{app:lambda_max_M} and~\ref{app:lambda_p_M}, where Appendix~\ref{app:lambda_max_M} derives the bounds for $\lambda_{\max}(M)$ and Appendix~\ref{app:lambda_p_M} derives the bounds for $\lambda_p(M)$. \hfill $\square$

Theorem~\ref{thm:tsf_condition} shows that, within the leading-order Hessian approximation, the effective condition number of KAN training in TSF depends explicitly on the spectrum of the input autocorrelation matrix \(R\).
 Compared with the independent-input case, temporal autocorrelation affects the Hessian spectrum through two mechanisms. First, the largest eigenvalue of $M$ is lower-bounded by $\lambda_{\max}(R)\lambda_{\max}(C)$, meaning that the dominant curvature direction becomes stronger when $R$ has a large principal eigenvalue. Second, the first non-degenerate eigenvalue $\lambda_p(M)$ is controlled by $\lambda_{\min}(R)$, so a small minimum eigenvalue of $R$ can substantially reduce the effective curvature in poorly conditioned directions.

When the lagged variables are independent, $R=I_p$, and hence
\begin{equation}
\lambda_{\max}(R)=\lambda_{\min}(R)=1.
\end{equation}
Then \eqref{eq:tsf_condition_bound} reduces to
\begin{equation}
1
\leq
\kappa_{\mathrm{TSF}}(M)
\leq
pC_k.
\end{equation}
This is consistent with the standard KAN conclusion that the effective condition number grows at most linearly with the input dimension and is independent of the spline grid size $G$.

By the way, the bounds in \eqref{eq:tsf_condition_bound} are consistent with each other. Since $R$ is a correlation matrix, its diagonal entries are all equal to one, and thus
\begin{equation}
\operatorname{Tr}(R)=p.
\end{equation}
Therefore, for positive semidefinite $R$, we have
\begin{equation}
\label{eq:app_lambda_R_leq_p}
\lambda_{\max}(R)\leq p.
\end{equation}
Moreover, since $C_k\geq 1$, it follows that
\begin{equation}
\frac{\lambda_{\max}(R)}
{2\lambda_{\min}(R)}
\leq
\frac{p}
{\lambda_{\min}(R)}
C_k.
\end{equation}
Hence, the lower bound in \eqref{eq:tsf_condition_bound} does not exceed the upper bound.

\subsection{Discussion}

The preceding bounds provide an optimization-based explanation for how temporal autocorrelation reintroduces spectral bias into KANs in the TSF setting. According to the Hessian/NTK-based explanation of spectral bias 
\cite{mlp-spectral-bias1, mlp-spectral-bias2}, a larger Hessian condition number indicates a more anisotropic optimization landscape: gradient-based optimization converges rapidly along eigendirections associated with 
large eigenvalues, but slowly along eigendirections associated with small eigenvalues. For lagged inputs generated from temporally autocorrelated time series, the dominant eigendirections of the autocorrelation matrix $R$ typically correspond to smooth and slowly varying patterns along the lag dimension, whereas small-eigenvalue directions are more closely associated with oscillatory and high-frequency patterns. Therefore, as temporal autocorrelation strengthens, the spectrum of $R$ becomes increasingly uneven, which in turn enlarges the separation between fast low-frequency learning directions and slow high-frequency learning directions in the KAN Hessian.

This observation highlights an important distinction between KANs under independent inputs and KANs applied to TSF. Although KANs are expected to avoid the severe grid-size-dependent ill-conditioning observed in conventional MLPs under the independent-input assumption, this favorable property can be weakened when lagged input variables are strongly correlated. In this case, the ill-conditioning of the autocorrelation matrix $R$ is transferred to the Hessian of the KAN training objective, making the optimization dynamics biased toward low-frequency components. Consequently, despite the favorable spectral behavior of KANs in the independent-input setting, standard KANs may still suffer from pronounced spectral bias when applied to highly autocorrelated time series.

\section{Experiments}

In this section, we empirically investigate the spectral behavior of KANs in time series forecasting. We first conduct controlled experiments on synthetic time series to verify whether KANs exhibit spectral bias when the lagged input variables are temporally autocorrelated. The results show that KANs indeed learn low-frequency components much faster than high-frequency components, and that this low-frequency preference becomes increasingly severe as the autocorrelation strength among input variables increases. We then introduce DCT preprocessing to reduce the correlation among lagged inputs. The experimental results demonstrate that DCT-KAN substantially alleviates the observed spectral bias and achieves better overall forecasting performance than the standard KAN. These findings further support our theoretical claim that the spectral bias of KANs in time series forecasting is mainly induced by the autocorrelation among input variables.

\subsection{Experimental Setup}

We conduct controlled experiments on synthetic autoregressive time series to study the influence of input autocorrelation on the spectral behavior of KANs. The univariate time series is generated by an AR($N$) process:
\begin{equation}
x_t
=
\rho_1 x_{t-1}
+
\rho_2\sum_{i=2}^{N}x_{t-i}
+
\sqrt{1-\rho_1^2-(N-1)\rho_2^2}\,\epsilon_t,
\qquad
\epsilon_t\sim\mathcal{N}(0,1),
\label{eq:exp_ar_process}
\end{equation}
where $\rho_1$ controls the main autocorrelation strength. The forecasting input is constructed as
\begin{equation}
\mathbf{X}_t
=
[x_t,x_{t-1},\ldots,x_{t-p+1}]^\top,
\qquad
p=3N.
\end{equation}
All inputs are standardized coordinate-wise using statistics computed from the training set.

To construct target functions with explicit frequency components, we use two projection directions:
\begin{equation}
\mathbf{w}_{\mathrm{easy}}
=
\mathbf{1}_{N}\otimes [1,1,1]^\top,
\qquad
\mathbf{w}_{\mathrm{hard}}
=
\mathbf{1}_{N}\otimes [1,-2,1]^\top,
\label{eq:projection_vectors}
\end{equation}

where $\otimes$ denotes the Kronecker product, so that the basic patterns $[1,1,1]$ and $[1,-2,1]$ are repeated $N$ times, respectively, and both vectors have dimension $p=3N$.
The vector $\mathbf{w}_{\mathrm{easy}}$ acts as a local averaging direction and mainly captures smooth, low-frequency variations along the lag dimension. In contrast, $\mathbf{w}_{\mathrm{hard}}$ is a repeated second-difference pattern, which suppresses smooth trends and emphasizes oscillatory lag structures. Therefore, these two directions allow us to separate easy low-frequency components from harder high-frequency components.

Let
\begin{equation}
V_{\mathrm{easy},t}
=
\mathbf{w}_{\mathrm{easy}}^\top \widetilde{\mathbf{X}}_t,
\qquad
V_{\mathrm{hard},t}
=
\mathbf{w}_{\mathrm{hard}}^\top \widetilde{\mathbf{X}}_t,
\end{equation}
where $\widetilde{\mathbf{X}}_t$ denotes the standardized input. Both projected variables are further normalized to zero mean and unit variance. The target is defined as
\begin{equation}
y_{t+1}
=
\sin(\omega_{\mathrm{low}}V_{\mathrm{easy},t})
+
\sin(\omega_{\mathrm{mid}}V_{\mathrm{hard},t})
+
\sin(\omega_{\mathrm{high}}V_{\mathrm{hard},t})
+
\eta_t,
\label{eq:exp_target}
\end{equation}
where $\eta_t\sim\mathcal{N}(0,0.05^2)$. The detailed settings are summarized in Table~\ref{tab:exp_setup}.

\begin{table}[t]
\centering
\caption{Experimental settings for the synthetic spectral-bias experiments.}
\label{tab:exp_setup}
\begin{tabularx}{\linewidth}{lX}
\toprule
\textbf{Item} & \textbf{Setting} \\
\midrule
Data-generating process &
AR($N$) process in \eqref{eq:exp_ar_process}. \\

Autocorrelation parameter &
$\rho_1\in\{0.1,0.2,\ldots,0.8\}$. \\

AR configurations &
$(N,p,\rho_2)\in\{(1,3,0),(2,6,0.1),(5,15,0.01)\}$. \\

Sequence length &
5000 samples. \\

Train/test split &
80\% / 20\%, split chronologically. \\

Number of runs &
100 independent runs for each $\rho_1$. \\

Target frequencies &
$\omega_{\mathrm{low}}=1$, $\omega_{\mathrm{mid}}=3$, $\omega_{\mathrm{high}}=6$. \\

Observation noise &
$\eta_t\sim\mathcal{N}(0,0.05^2)$. \\

Model &
FastKAN with $\text{width}=[p,50,1]$, $\text{grid}=8$. \\

Optimizer &
Adam with learning rate $0.005$. \\

Batch size / epochs &
256 / 150. \\

KAN input &
Standardized lagged vector $\widetilde{\mathbf{X}}_t$. \\

DCT-KAN input &
Orthonormal DCT of $\widetilde{\mathbf{X}}_t$ along the lag dimension, followed by coordinate-wise standardization using training-set statistics. \\

Evaluation metrics &
Test MSE and component-wise recovery errors for low-, middle-, and high-frequency components. \\
\bottomrule
\end{tabularx}
\end{table}

For component-wise evaluation, we use the known synthetic components. Given a prediction $\hat{y}$ and a true component $C_q$, where $q\in\{\mathrm{low},\mathrm{mid},\mathrm{high}\}$, the recovered amplitude is estimated by
\begin{equation}
\hat{a}_q
=
\operatorname{clip}
\left(
\frac{\operatorname{Cov}(\hat{y},C_q)}
{\operatorname{Var}(C_q)+10^{-8}},
0,1
\right),
\end{equation}
and the corresponding component error is
\begin{equation}
E_q=(1-\hat{a}_q)^2.
\end{equation}
A smaller $E_q$ indicates better recovery of the corresponding frequency component.

\subsection{KANs Exhibit Spectral Bias in Time Series Forecasting}

\begin{figure*}[tbp]
    \centering

    \subfloat[KAN MSE]{
        \includegraphics[width=0.48\textwidth]{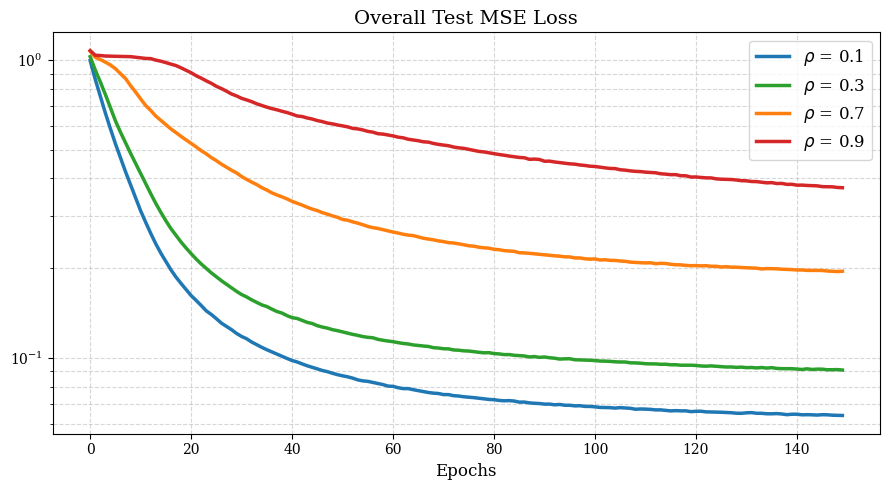}
        \label{fig:epoch-mse}
    }
    \subfloat[DCT-KAN MSE]{
        \includegraphics[width=0.48\textwidth]{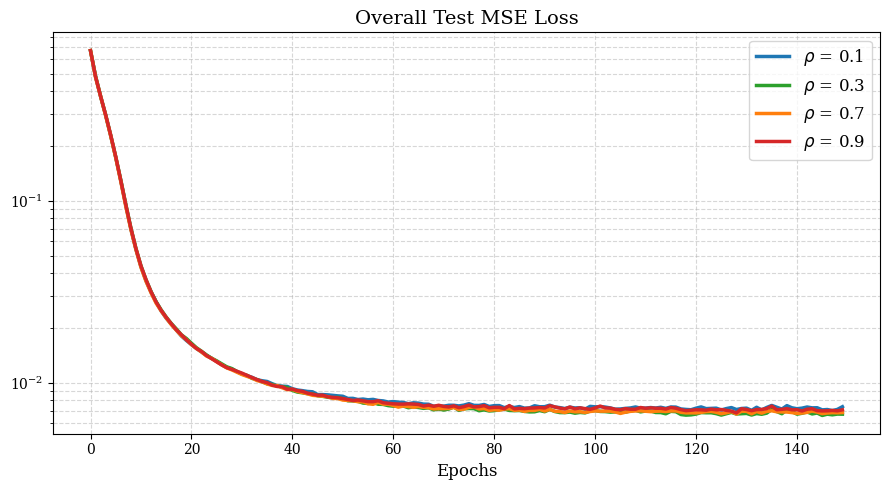}
        \label{fig:epoch-dct-mse}
    }

    \subfloat[KAN Freq. Error]{
        \includegraphics[width=0.48\textwidth]{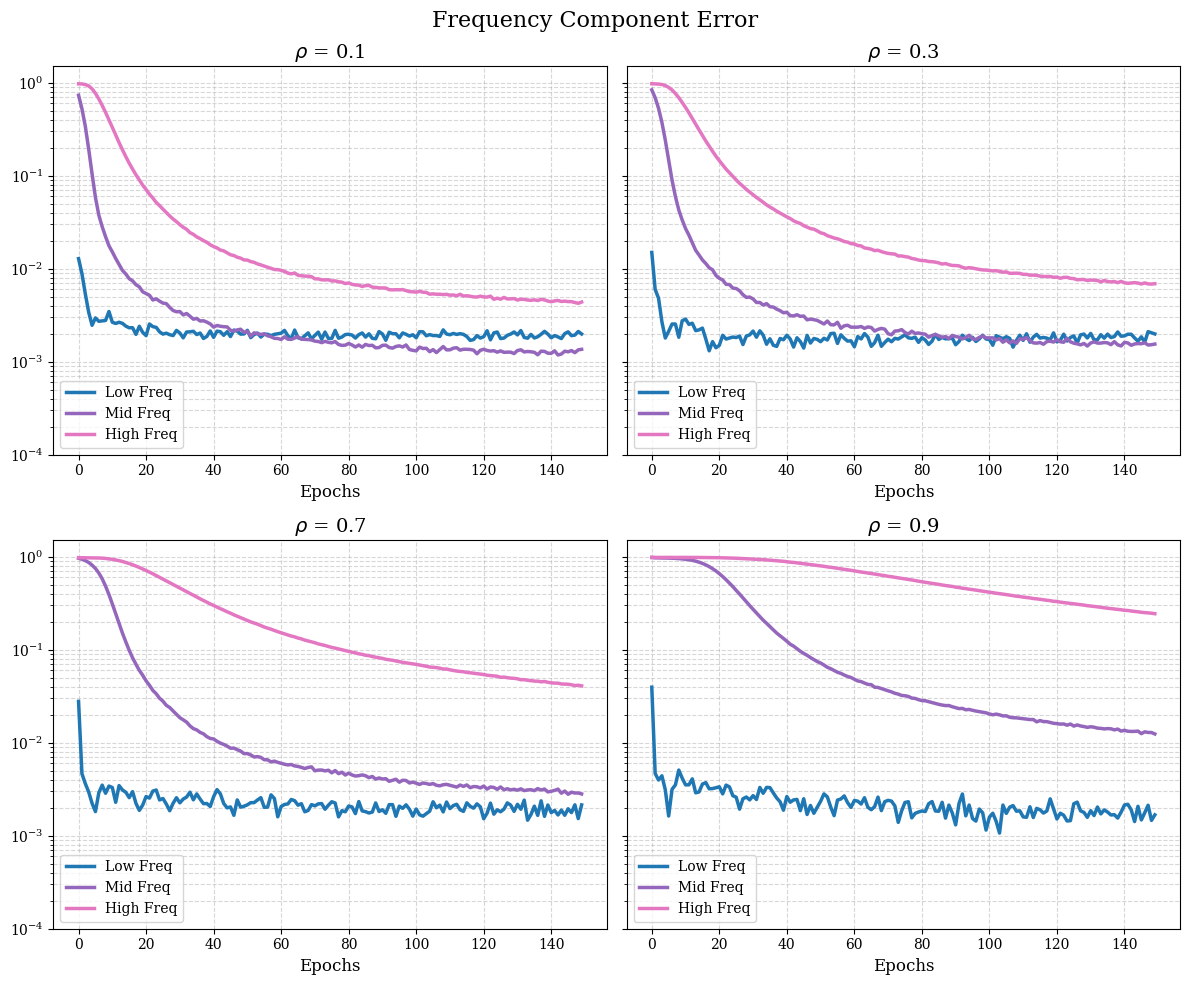}
        \label{fig:epoch-freq}
    }
    \subfloat[DCT-KAN Freq. Error]{
        \includegraphics[width=0.48\textwidth]{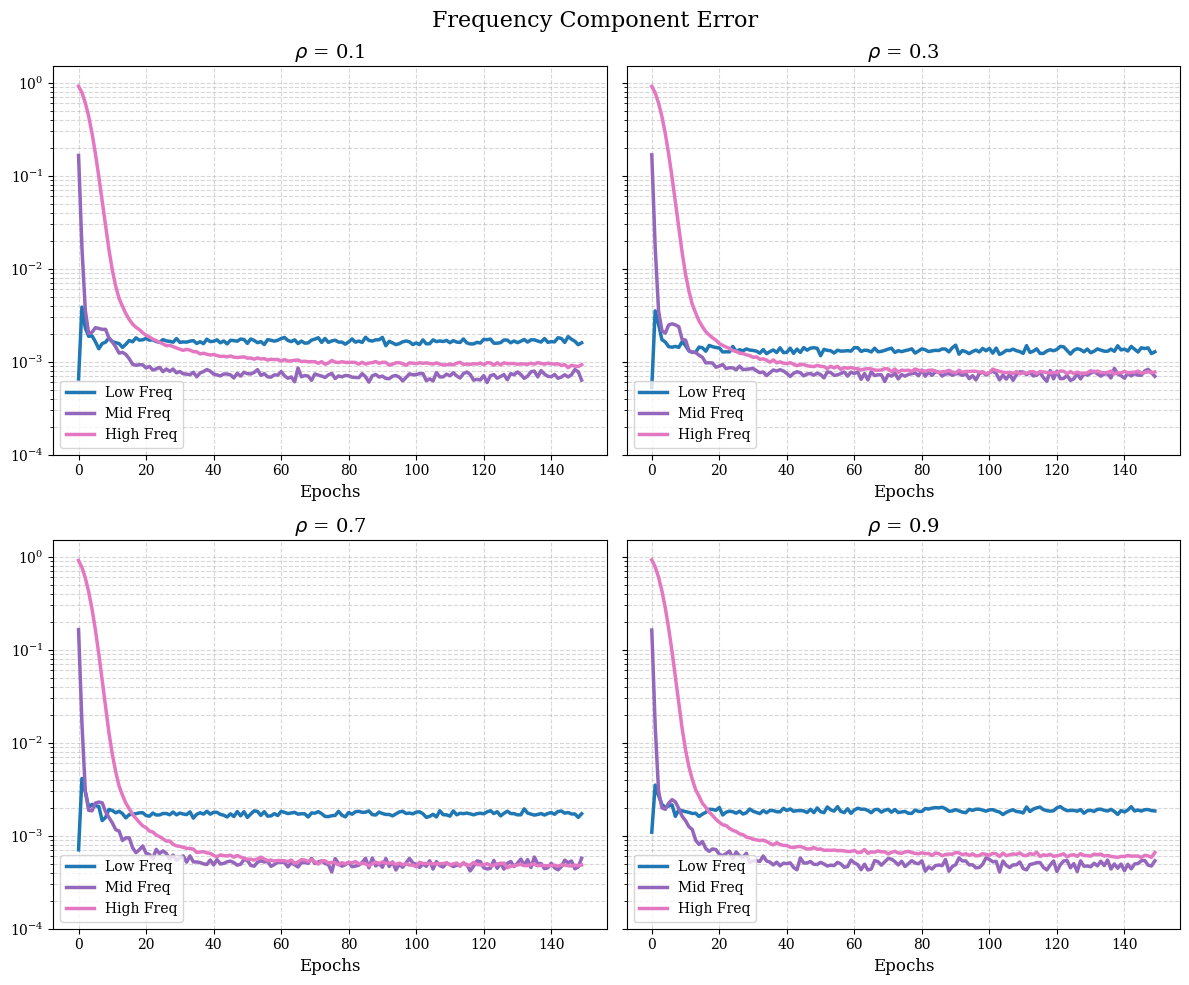}
        \label{fig:epoch-dct-freq}
    }

    \caption{Epoch-wise training dynamics of KAN and DCT-KAN ($N=1$).}
    \label{fig:epoch}
\end{figure*}

\begin{figure*}[tbp]
    \centering

    \subfloat[MSE ($N=1$)]{
        \includegraphics[width=0.33\textwidth]{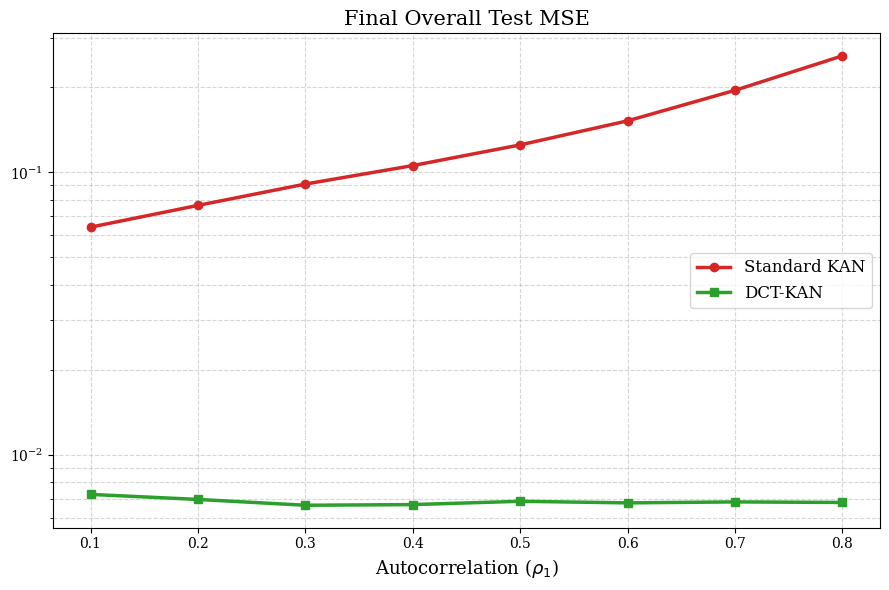}
        \label{fig:rho-mse-1}
    }
    \subfloat[MSE ($N=2$)]{
        \includegraphics[width=0.33\textwidth]{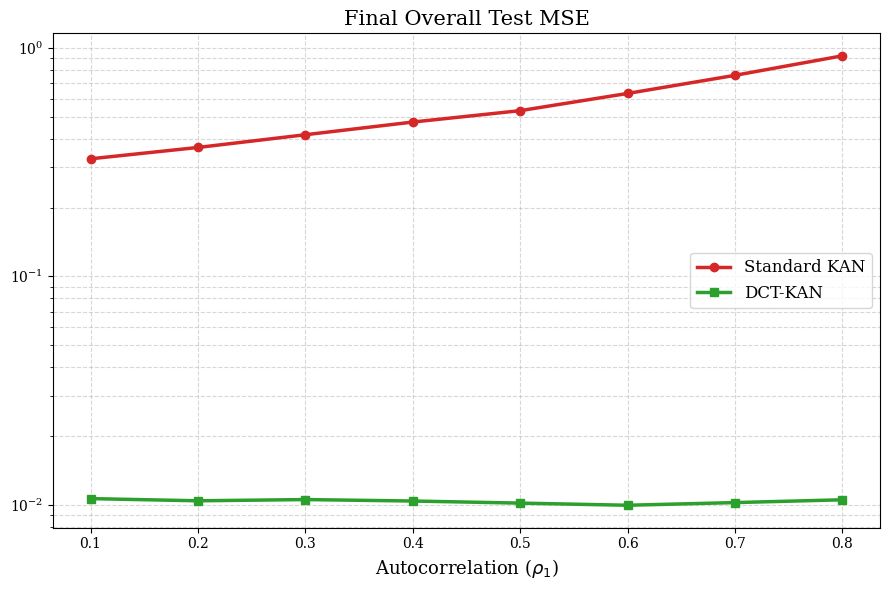}
        \label{fig:rho-mse-2}
    }
    \subfloat[MSE ($N=5$)]{
        \includegraphics[width=0.33\textwidth]{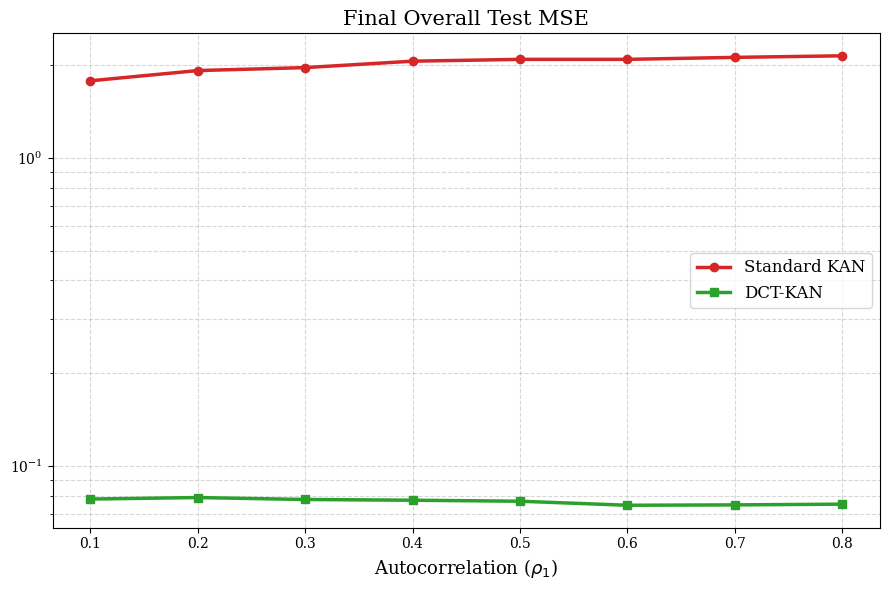}
        \label{fig:rho-mse-5}
    }

    \subfloat[Freq. Error ($N=1$)]{
        \includegraphics[width=0.33\textwidth]{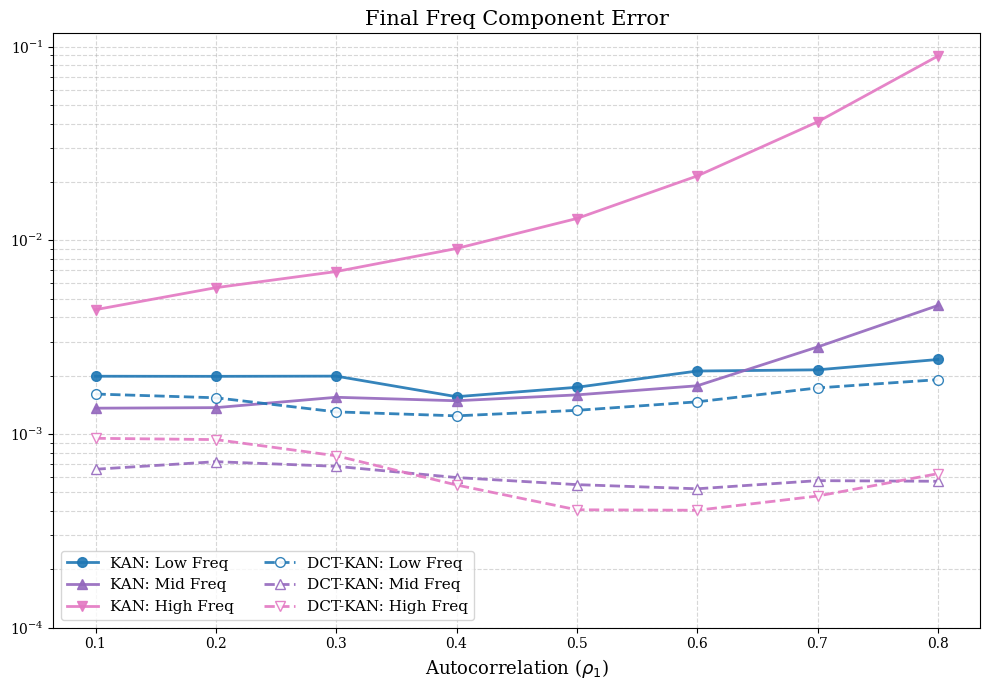}
        \label{fig:rho-freq-1}
    }
    \subfloat[Freq. Error ($N=2$)]{
        \includegraphics[width=0.33\textwidth]{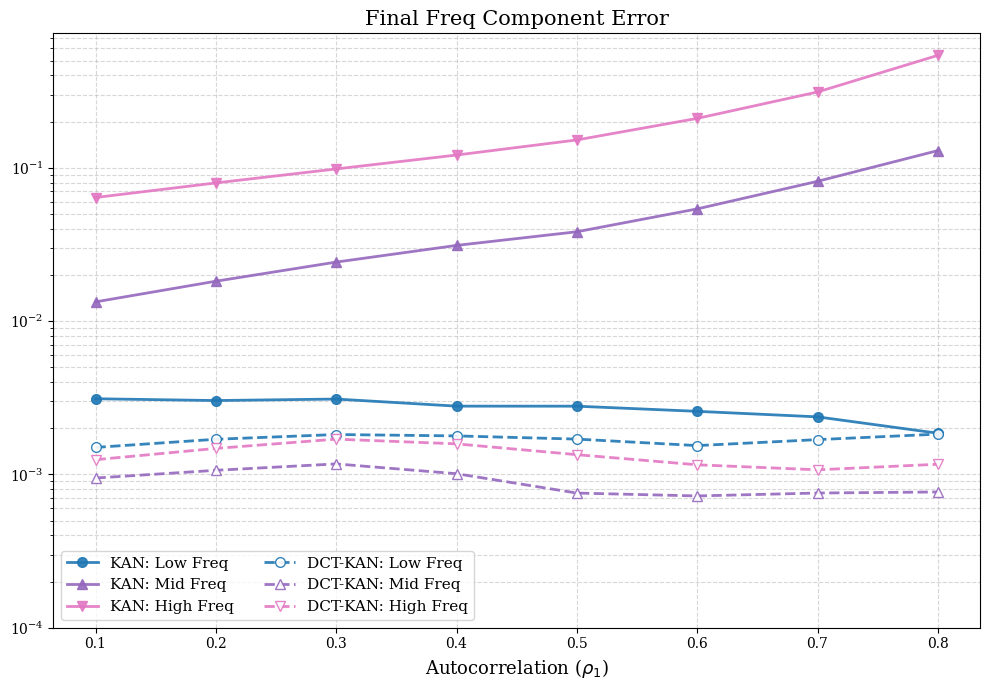}
        \label{fig:rho-freq-2}
    }
    \subfloat[Freq. Error ($N=5$)]{
        \includegraphics[width=0.33\textwidth]{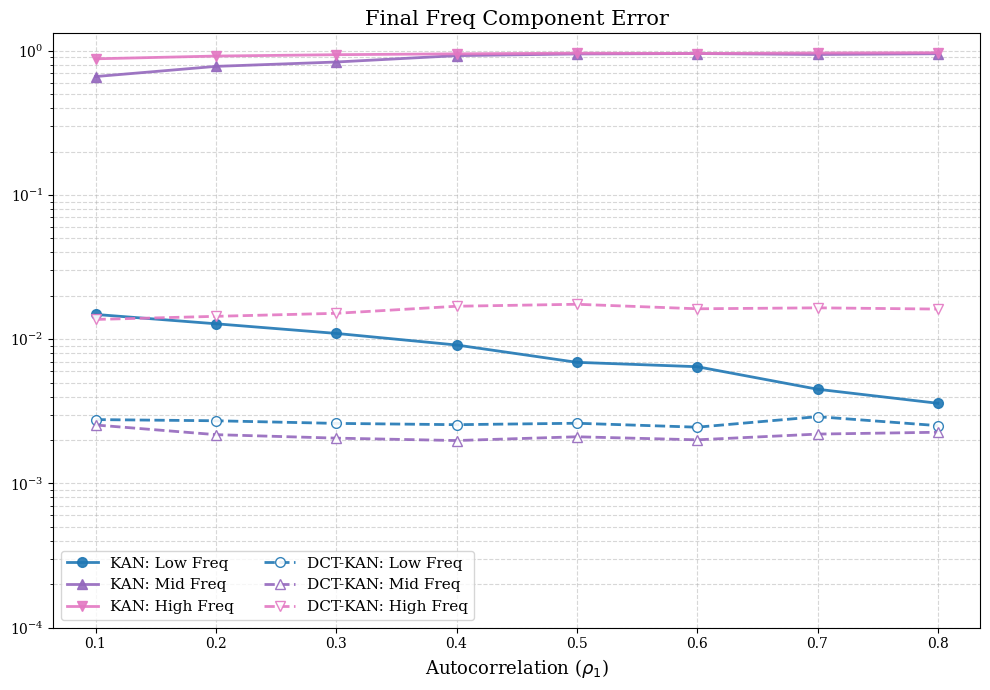}
        \label{fig:rho-freq-5}
    }

    \caption{Autocorrelation-wise performance of KAN and DCT-KAN.}
    \label{fig:rho}
\end{figure*}

We first examine the training dynamics of the standard KAN under different autocorrelation strengths. Figure~\ref{fig:epoch-mse} reports the epoch-wise test MSE for $N=1$. It can be observed that, although the training loss decreases for all values of $\rho_1$, the final converged MSE becomes increasingly large as $\rho_1$ increases. This indicates that stronger temporal autocorrelation makes the forecasting problem substantially more difficult for KANs. More importantly, the component-wise errors in Figure~\ref{fig:epoch-freq} reveal a clear frequency-dependent learning behavior. The low-frequency component is learned rapidly and its recovery error decreases quickly during training, whereas the middle- and high-frequency components converge much more slowly. When $\rho_1$ is large, the high-frequency error remains high even after sufficient training epochs, showing that KANs have difficulty recovering oscillatory components from strongly autocorrelated lagged inputs.

We further investigate this phenomenon by varying the autocorrelation strength more systematically. As shown in Figures~\ref{fig:rho-mse-1}--\ref{fig:rho-mse-5}, the test MSE of KAN consistently increases as $\rho_1$ grows. This trend is observed across all three AR configurations, suggesting that the degradation is not specific to a particular input dimension or AR order. The frequency-wise results in Figures~\ref{fig:rho-freq-1}--\ref{fig:rho-freq-5} provide a more detailed explanation: the recovery error of the low-frequency component remains relatively small, while the errors of the middle- and high-frequency components increase rapidly with $\rho_1$. Therefore, the performance degradation under stronger autocorrelation is mainly caused by the failure to learn higher-frequency components rather than by a uniform deterioration across all components.

The effect becomes even more pronounced as $N$ increases. A larger $N$ corresponds to a longer and more complex lagged input vector, which introduces richer correlation structures among input coordinates. Comparing the results for $N=1$, $N=2$, and $N=5$ in Figure~\ref{fig:rho}, we observe that the spectral bias of KANs becomes progressively stronger as the input sequence becomes more complex. In particular, when $N=5$, KAN almost fails to recover the middle- and high-frequency components, as indicated by their persistently large component errors in Figure~\ref{fig:rho-freq-5}. These empirical observations are consistent with the theoretical analysis in Section~III: stronger temporal autocorrelation makes the autocorrelation matrix $R$ more ill-conditioned, enlarges the effective Hessian condition number, and consequently slows down optimization along directions associated with oscillatory high-frequency patterns.

\subsection{DCT-KAN Alleviates Spectral Bias}

We next examine whether decorrelating the lagged inputs can mitigate the spectral bias observed above. Figure~\ref{fig:epoch-dct-mse} shows the epoch-wise MSE curves of DCT-KAN for $N=1$ under different values of $\rho_1$. In contrast to the standard KAN in Figure~\ref{fig:epoch-mse}, the MSE curves of DCT-KAN are almost indistinguishable across different autocorrelation strengths. Moreover, the final converged MSE is substantially lower than that of the standard KAN, reaching approximately an order-of-magnitude improvement. This indicates that, after DCT preprocessing, the optimization and generalization behavior of KAN becomes much less sensitive to the temporal autocorrelation of the original lagged inputs.
﻿
The component-wise results further support this conclusion. As shown in Figure~\ref{fig:epoch-dct-freq}, DCT-KAN learns the low-, middle-, and high-frequency components much more effectively than the standard KAN. In particular, the recovery errors of the middle- and high-frequency components decrease significantly during training and remain at a low level even when $\rho_1$ is large. Compared with Figure~\ref{fig:epoch-freq}, the gap between the low-frequency component and the higher-frequency components is greatly reduced. This suggests that DCT preprocessing weakens the low-frequency preference of KANs by reducing the correlation among input coordinates.
﻿
The autocorrelation-wise results in Figure~\ref{fig:rho} provide a more systematic comparison. Across all three AR configurations, DCT-KAN consistently achieves much smaller test MSE than the standard KAN, as shown in Figures~\ref{fig:rho-mse-1}--\ref{fig:rho-mse-5}. More importantly, its MSE remains nearly stable as $\rho_1$ increases, whereas the MSE of the standard KAN grows rapidly with stronger autocorrelation. The same pattern can be observed in the component-wise errors in Figures~\ref{fig:rho-freq-1}--\ref{fig:rho-freq-5}: DCT-KAN maintains low recovery errors for all frequency components, and the errors exhibit only weak dependence on $\rho_1$. Even in the most challenging case of $N=5$, where the standard KAN almost fails to learn the middle- and high-frequency components, DCT-KAN still recovers these components effectively. These results confirm that the spectral bias observed in standard KANs for TSF is largely induced by the autocorrelation among lagged inputs, and that DCT preprocessing provides an effective way to alleviate this problem.
﻿

\section{Conclusion}

In this paper, we pointed out that the favorable spectral behavior of KANs under independent inputs does not directly extend to time series forecasting, where lagged input variables are inherently autocorrelated. Through a leading-order Hessian analysis based on a covariance-transfer approximation, we showed that temporal autocorrelation can alter the Hessian spectrum of KAN training and make the effective condition number depend on the extreme eigenvalues of the input autocorrelation matrix. As the autocorrelation strength increases, the optimization landscape becomes more anisotropic, causing KANs to learn low-frequency components much faster than middle- and high-frequency components. Controlled experiments on synthetic time series supported this prediction and showed that the spectral bias becomes increasingly severe under stronger autocorrelation. This conclusion suggests that standard KANs may face substantial difficulties when handling highly autocorrelated time series. To address this issue, we introduced DCT preprocessing to reduce correlations among lagged inputs. The resulting DCT-KAN substantially alleviates the low-frequency preference and improves the recovery of high-frequency components, indicating that the spectral bias observed in KANs for time series forecasting is mainly induced by input autocorrelation rather than being an intrinsic limitation of the KAN architecture.

\appendices

\section{Derivation of Proposition~\ref{prop:tsf_hessian}}
\label{app:derivation_tsf_hessian}

In this appendix, we derive the leading-order Hessian approximation used in Proposition~\ref{prop:tsf_hessian}. The derivation consists of two parts. First, we obtain an exact expression for the Hessian in terms of the cross-covariance matrices of the B-spline-transformed lagged inputs. Second, we introduce the covariance-transfer closure in Assumption~\ref{ass:cov_transfer}. Therefore, the Kronecker-product form in Proposition~\ref{prop:tsf_hessian} should be interpreted as a leading-order analytical approximation rather than as an exact Hessian identity for arbitrary autocorrelated inputs.

Recall that the single-layer KAN considered in the main text is
\begin{equation}
f_{\boldsymbol{\theta}}(\mathbf{X}_t)
=
\sum_{j=1}^{p}
\sum_{\ell=1}^{m}
c_{j\ell}B_{\ell}(X_{t,j}),
\qquad
m=G+k-1 ,
\label{eq:app_kan_model}
\end{equation}
where $p$ is the look-back window length and $m$ is the number of B-spline basis functions. The population MSE loss is
\begin{equation}
\mathcal{L}(\boldsymbol{\theta})
=
\frac{1}{2}
\mathbb{E}
\left[
\left(
y_{t+1}
-
f_{\boldsymbol{\theta}}(\mathbf{X}_t)
\right)^2
\right].
\label{eq:app_loss}
\end{equation}
Since the B-spline basis functions are bounded on the compact input domain considered in this paper, differentiation can be interchanged with expectation.

Let
\begin{equation}
e_t(\boldsymbol{\theta})
=
y_{t+1}
-
f_{\boldsymbol{\theta}}(\mathbf{X}_t)
\end{equation}
be the prediction error. Since the model in \eqref{eq:app_kan_model} is linear with respect to the spline coefficients $c_{j\ell}$, we have
\begin{equation}
\frac{\partial f_{\boldsymbol{\theta}}(\mathbf{X}_t)}
{\partial c_{j\ell}}
=
B_{\ell}(X_{t,j}).
\label{eq:app_partial_model}
\end{equation}
Therefore,
\begin{align}
\frac{\partial \mathcal{L}}
{\partial c_{j\ell}}
&=
\frac{\partial}{\partial c_{j\ell}}
\frac{1}{2}
\mathbb{E}
\left[
e_t(\boldsymbol{\theta})^2
\right]  \notag \\
&=
\mathbb{E}
\left[
e_t(\boldsymbol{\theta})
\frac{\partial e_t(\boldsymbol{\theta})}
{\partial c_{j\ell}}
\right]  \notag \\
&=
-
\mathbb{E}
\left[
\left(
y_{t+1}
-
f_{\boldsymbol{\theta}}(\mathbf{X}_t)
\right)
B_{\ell}(X_{t,j})
\right]  \notag \\
&=
\mathbb{E}
\left[
\left(
f_{\boldsymbol{\theta}}(\mathbf{X}_t)
-
y_{t+1}
\right)
B_{\ell}(X_{t,j})
\right].
\label{eq:app_first_derivative}
\end{align}
Taking another derivative with respect to $c_{j'\ell'}$ gives
\begin{align}
\frac{\partial^2 \mathcal{L}}
{\partial c_{j\ell}\partial c_{j'\ell'}}
&=
\mathbb{E}
\left[
\frac{\partial f_{\boldsymbol{\theta}}(\mathbf{X}_t)}
{\partial c_{j'\ell'}}
B_{\ell}(X_{t,j})
\right]  \notag \\
&=
\mathbb{E}
\left[
B_{\ell}(X_{t,j})
B_{\ell'}(X_{t,j'})
\right].
\label{eq:app_second_derivative}
\end{align}
Thus, the Hessian is independent of the target value $y_{t+1}$ and is determined only by the second-order statistics of the B-spline-transformed inputs.

Define the B-spline basis vector
\begin{equation}
\mathbf{b}(z)
=
[B_1(z),B_2(z),\ldots,B_m(z)]^\top .
\end{equation}
For the $j$-th lagged input variable, denote
\begin{equation}
\mathbf{b}_j
=
\mathbf{b}(X_{t,j}).
\end{equation}
Then the $(j,j')$ block of the Hessian can be written exactly as
\begin{equation}
M_{jj'}
=
\mathbb{E}
\left[
\mathbf{b}_j\mathbf{b}_{j'}^\top
\right],
\label{eq:app_hessian_block_raw}
\end{equation}
where the $(\ell,\ell')$ entry of $M_{jj'}$ is
$\mathbb{E}[B_{\ell}(X_{t,j})B_{\ell'}(X_{t,j'})]$.

By stationarity and marginal normalization, all lagged input coordinates are assumed to have the same marginal distribution. Hence, the mean vector
\begin{equation}
\boldsymbol{\nu}
=
\mathbb{E}[\mathbf{b}_j]
\end{equation}
does not depend on $j$. Decompose each transformed input vector into its mean and centered fluctuation:
\begin{equation}
\mathbf{b}_j
=
\boldsymbol{\nu}
+
\widetilde{\mathbf{b}}_j,
\qquad
\mathbb{E}[\widetilde{\mathbf{b}}_j]=\mathbf{0}.
\label{eq:app_mean_fluctuation}
\end{equation}
Substituting \eqref{eq:app_mean_fluctuation} into \eqref{eq:app_hessian_block_raw}, we obtain
\begin{align}
M_{jj'}
&=
\mathbb{E}
\left[
(\boldsymbol{\nu}+\widetilde{\mathbf{b}}_j)
(\boldsymbol{\nu}+\widetilde{\mathbf{b}}_{j'})^\top
\right]  \notag \\
&=
\boldsymbol{\nu}\boldsymbol{\nu}^\top
+
\mathbb{E}
\left[
\widetilde{\mathbf{b}}_j
\widetilde{\mathbf{b}}_{j'}^\top
\right].
\label{eq:app_block_mean_cov}
\end{align}
Define
\begin{equation}
D
=
\boldsymbol{\nu}\boldsymbol{\nu}^\top ,
\label{eq:app_D_definition}
\end{equation}
and
\begin{equation}
C
=
\mathbb{E}
\left[
\mathbf{b}_j\mathbf{b}_j^\top
\right].
\label{eq:app_C_definition}
\end{equation}
Then the covariance matrix of the centered B-spline basis vector is
\begin{align}
S
&=
\operatorname{Cov}(\mathbf{b}_j)  \notag \\
&=
\mathbb{E}
\left[
(\mathbf{b}_j-\boldsymbol{\nu})
(\mathbf{b}_j-\boldsymbol{\nu})^\top
\right]  \notag \\
&=
C-D .
\label{eq:app_S_definition}
\end{align}
For $j=j'$, we have
\begin{equation}
M_{jj}
=
D+S
=
C.
\label{eq:app_diagonal_block}
\end{equation}

For $j\neq j'$, define the exact cross-covariance matrix
\begin{equation}
\Gamma_{jj'}
=
\operatorname{Cov}
\left(
\mathbf{b}(X_{t,j}),
\mathbf{b}(X_{t,j'})
\right)
=
\mathbb{E}
\left[
\widetilde{\mathbf{b}}_j
\widetilde{\mathbf{b}}_{j'}^\top
\right].
\label{eq:app_exact_cross_cov}
\end{equation}
Then the exact Hessian block is
\begin{equation}
M_{jj'}
=
D+\Gamma_{jj'}.
\label{eq:app_exact_block}
\end{equation}
Equation~\eqref{eq:app_exact_block} is an exact identity. In general, however, for nonlinear basis functions $B_\ell(\cdot)$, the matrix $\Gamma_{jj'}$ is not determined solely by the scalar correlation coefficient
\begin{equation}
R_{jj'}
=
\operatorname{Corr}(X_{t,j},X_{t,j'}).
\end{equation}
It also depends on the full joint distribution of $(X_{t,j},X_{t,j'})$ and on the nonlinear shape of the basis functions.

The first-order covariance-transfer approximation used in Proposition~\ref{prop:tsf_hessian} replaces the exact cross-covariance matrix $\Gamma_{jj'}$ by a scalar-correlation rescaling of the marginal covariance matrix $S=C-D$:
\begin{equation}
\Gamma_{jj'}
\approx
R_{jj'}(C-D).
\label{eq:app_cov_transfer_approx}
\end{equation}
Equivalently, one may write the exact cross-covariance as
\begin{equation}
\Gamma_{jj'}
=
R_{jj'}(C-D)
+
E_{jj'},
\label{eq:app_cov_transfer_residual}
\end{equation}
where
\begin{equation}
E_{jj'}
=
\Gamma_{jj'}-R_{jj'}(C-D)
\label{eq:app_residual_definition}
\end{equation}
is the residual matrix associated with the covariance-transfer approximation. By construction, $E_{jj}=0$, because $R_{jj}=1$ and $\Gamma_{jj}=C-D$.

Thus, the exact Hessian can be written as
\begin{equation}
M_{jj'}
=
D
+
R_{jj'}(C-D)
+
E_{jj'}.
\label{eq:app_hessian_block_with_residual}
\end{equation}
Let $E$ denote the block matrix whose $(j,j')$ block is $E_{jj'}$. Then the exact Hessian admits the decomposition
\begin{equation}
M
=
J_p\otimes D
+
R\otimes(C-D)
+
E,
\label{eq:app_exact_hessian_with_residual}
\end{equation}
where $J_p=\mathbf{1}_p\mathbf{1}_p^\top$.

The approximation in Proposition~\ref{prop:tsf_hessian} corresponds to neglecting the residual matrix $E$, namely,
\begin{equation}
E\approx 0.
\end{equation}
Under this first-order covariance-transfer approximation, \eqref{eq:app_hessian_block_with_residual} reduces to
\begin{equation}
M_{jj'}
\approx
D
+
R_{jj'}(C-D),
\qquad
j,j'=1,\ldots,p,
\label{eq:app_hessian_block_final}
\end{equation}
and the full Hessian becomes
\begin{equation}
M
\approx
J_p\otimes D
+
R\otimes(C-D).
\label{eq:app_hessian_kron_final}
\end{equation}
This is the Kronecker-product representation stated in Proposition~\ref{prop:tsf_hessian}.

We emphasize that \eqref{eq:app_cov_transfer_approx} is not a universal identity for arbitrary nonlinear transformations or arbitrary joint distributions. It is a second-order closure approximation: the shape of the cross-covariance matrix of the transformed variables is approximated by the marginal covariance matrix $C-D$, while its magnitude is controlled by the original lag correlation coefficient $R_{jj'}$. The approximation is exact on the diagonal blocks. It is also consistent with the independent-input case: if $X_{t,j}$ and $X_{t,j'}$ are independent for $j\neq j'$, then $\Gamma_{jj'}=0$ and $R_{jj'}=0$, so the off-diagonal residuals vanish.

If one keeps the residual term, then the approximation error is explicitly represented by $E$ in \eqref{eq:app_exact_hessian_with_residual}. In particular, if $\|E\|_2\leq \varepsilon$, then Weyl's inequality gives
\begin{equation}
\left|
\lambda_i(M)
-
\lambda_i\left(
J_p\otimes D+R\otimes(C-D)
\right)
\right|
\leq
\varepsilon
\end{equation}
for every eigenvalue index $i$. Therefore, the Kronecker-product Hessian in \eqref{eq:app_hessian_kron_final} should be understood as the leading-order Hessian, and the residual matrix $E$ measures the deviation of the exact nonlinear transformed-input covariance from the first-order covariance-transfer approximation.

\section{Bounds for $\lambda_{\max}(M)$}
\label{app:lambda_max_M}

In this appendix, we derive the upper and lower bounds of the largest eigenvalue of the Hessian matrix $M$ used in Theorem~\ref{thm:tsf_condition}. Recall from Proposition~\ref{prop:tsf_hessian} that, under the covariance-transfer approximation, the Hessian has the Kronecker-product form
\begin{equation}
M
\approx
J_p\otimes D
+
R\otimes(C-D),
\label{eq:app_lambda_max_M}
\end{equation}
where $J_p$ is the $p\times p$ all-one matrix, $R$ is the autocorrelation matrix of the lagged inputs, $C=\mathbb{E}[\mathbf{b}(z)\mathbf{b}(z)^{\mathrm{T}}]$, and $D=\boldsymbol{\nu}\boldsymbol{\nu}^{\mathrm{T}}$ with $\boldsymbol{\nu}=\mathbb{E}[\mathbf{b}(z)]$.

For notational simplicity, define
\begin{equation}
S=C-D.
\label{eq:app_S_definition_lambda_max}
\end{equation}
Since $S$ is the covariance matrix of the centered B-spline basis vector, it is positive semidefinite. Thus,
\begin{equation}
M
\approx
J_p\otimes D
+
R\otimes S.
\label{eq:app_M_JD_RS_lambda_max}
\end{equation}

\subsection{Upper Bound of $\lambda_{\max}(M)$}

Let
\begin{equation}
\mathbf{z}
=
(\mathbf{z}_1,\mathbf{z}_2,\ldots,\mathbf{z}_p),
\end{equation}
where each $\mathbf{z}_j\in\mathbb{R}^{m}$. By the definition of the Hessian entries, the quadratic form of $M$ can be written as
\begin{align}
\langle \mathbf{z},M\mathbf{z}\rangle
&=
\sum_{j=1}^{p}
\sum_{j'=1}^{p}
\mathbb{E}
\left[
\langle \mathbf{z}_j,\mathbf{b}(X_{t,j})\rangle
\langle \mathbf{z}_{j'},\mathbf{b}(X_{t,j'})\rangle
\right] \notag \\
&=
\mathbb{E}
\left[
\left(
\sum_{j=1}^{p}
\langle \mathbf{z}_j,\mathbf{b}(X_{t,j})\rangle
\right)^2
\right].
\label{eq:app_quadratic_expectation_lambda_max}
\end{align}
For any real numbers $a_1,\ldots,a_p$, the Cauchy--Schwarz inequality gives
\begin{equation}
\left(
\sum_{j=1}^{p}a_j
\right)^2
\leq
p\sum_{j=1}^{p}a_j^2.
\end{equation}
Taking
\begin{equation}
a_j=
\langle \mathbf{z}_j,\mathbf{b}(X_{t,j})\rangle,
\end{equation}
we obtain
\begin{align}
\langle \mathbf{z},M\mathbf{z}\rangle
&\leq
p
\sum_{j=1}^{p}
\mathbb{E}
\left[
\langle \mathbf{z}_j,\mathbf{b}(X_{t,j})\rangle^2
\right] \notag \\
&=
p
\sum_{j=1}^{p}
\langle \mathbf{z}_j,C\mathbf{z}_j\rangle .
\label{eq:app_upper_step1_lambda_max}
\end{align}
Since
\begin{equation}
\langle \mathbf{z}_j,C\mathbf{z}_j\rangle
\leq
\lambda_{\max}(C)\|\mathbf{z}_j\|_2^2,
\end{equation}
we have
\begin{align}
\langle \mathbf{z},M\mathbf{z}\rangle
&\leq
p\lambda_{\max}(C)
\sum_{j=1}^{p}
\|\mathbf{z}_j\|_2^2 \notag \\
&=
p\lambda_{\max}(C)
\|\mathbf{z}\|_2^2 .
\label{eq:app_upper_step2_lambda_max}
\end{align}
Therefore, by the Rayleigh--Ritz theorem,
\begin{equation}
\lambda_{\max}(M)
\leq
p\lambda_{\max}(C).
\label{eq:app_lambda_max_upper_final}
\end{equation}

\subsection{Lower Bound of $\lambda_{\max}(M)$}

We next derive a lower bound that does not rely on the entrywise nonnegativity of $R$. Since both $D$ and $S$ are positive semidefinite, and since both $J_p$ and $R$ are also positive semidefinite, the two Kronecker-product terms in \eqref{eq:app_M_JD_RS_lambda_max} are positive semidefinite. Hence, removing either term cannot increase the largest eigenvalue. Therefore,
\begin{equation}
\lambda_{\max}(M)
\geq
\lambda_{\max}(J_p\otimes D).
\label{eq:app_lower_JD}
\end{equation}
Because the largest eigenvalue of $J_p$ is $p$, we have
\begin{equation}
\lambda_{\max}(J_p\otimes D)
=
p\lambda_{\max}(D).
\end{equation}
Thus,
\begin{equation}
\lambda_{\max}(M)
\geq
p\lambda_{\max}(D).
\label{eq:app_lower_D}
\end{equation}
Similarly,
\begin{equation}
\lambda_{\max}(M)
\geq
\lambda_{\max}(R\otimes S).
\label{eq:app_lower_RS}
\end{equation}
Since the eigenvalues of a Kronecker product are pairwise products of the eigenvalues of its two factors, we obtain
\begin{equation}
\lambda_{\max}(R\otimes S)
=
\lambda_{\max}(R)\lambda_{\max}(S).
\end{equation}
Therefore,
\begin{equation}
\lambda_{\max}(M)
\geq
\lambda_{\max}(R)\lambda_{\max}(S).
\label{eq:app_lower_S}
\end{equation}

Now, using \eqref{eq:app_lambda_R_leq_p}, the bound in \eqref{eq:app_lower_D} implies
\begin{equation}
\lambda_{\max}(M)
\geq
\lambda_{\max}(R)\lambda_{\max}(D).
\label{eq:app_lower_D_rescaled}
\end{equation}
Combining \eqref{eq:app_lower_D_rescaled} and \eqref{eq:app_lower_S}, we get
\begin{equation}
2\lambda_{\max}(M)
\geq
\lambda_{\max}(R)
\left(
\lambda_{\max}(D)+\lambda_{\max}(S)
\right).
\label{eq:app_lower_sum_DS}
\end{equation}
Since
\begin{equation}
C=D+S,
\end{equation}
the largest eigenvalue of $C$ satisfies
\begin{equation}
\lambda_{\max}(C)
\leq
\lambda_{\max}(D)+\lambda_{\max}(S).
\label{eq:app_C_DS_subadditive}
\end{equation}
Substituting \eqref{eq:app_C_DS_subadditive} into \eqref{eq:app_lower_sum_DS} yields
\begin{equation}
2\lambda_{\max}(M)
\geq
\lambda_{\max}(R)\lambda_{\max}(C).
\end{equation}
Hence,
\begin{equation}
\lambda_{\max}(M)
\geq
\frac{1}{2}
\lambda_{\max}(R)\lambda_{\max}(C).
\label{eq:app_lambda_max_lower_final}
\end{equation}

Combining \eqref{eq:app_lambda_max_upper_final} and \eqref{eq:app_lambda_max_lower_final}, we obtain
\begin{equation}
\frac{1}{2}
\lambda_{\max}(R)\lambda_{\max}(C)
\leq
\lambda_{\max}(M)
\leq
p\lambda_{\max}(C),
\end{equation}
which completes the proof.

This result shows that, even without assuming entrywise positive autocorrelation, the largest curvature of the KAN training objective is still controlled by the largest eigenvalue of the input autocorrelation matrix $R$. Therefore, when the lagged-input correlation structure becomes more concentrated in a dominant direction, the largest eigenvalue of the Hessian also increases at least proportionally up to a universal constant factor.

\section{Bounds for $\lambda_p(M)$}
\label{app:lambda_p_M}

In this appendix, we derive the lower and upper bounds of the first non-degenerate eigenvalue $\lambda_p(M)$ used in Theorem~\ref{thm:tsf_condition}. Throughout this appendix, $M$ denotes the leading-order Hessian under the covariance-transfer approximation in Proposition~\ref{prop:tsf_hessian}, namely
\begin{equation}
M
\approx
J_p\otimes D
+
R\otimes(C-D),
\label{eq:app_lambdap_M}
\end{equation}
where $J_p=\mathbf{1}_p\mathbf{1}_p^\top$, $R$ is the autocorrelation matrix of the lagged inputs, $C$ is the B-spline Gram matrix, and $D=\boldsymbol{\nu}\boldsymbol{\nu}^\top$ with
\begin{equation}
\boldsymbol{\nu}
=
\mathbb{E}[\mathbf{b}(z)].
\end{equation}
For notational simplicity, define
\begin{equation}
S=C-D .
\end{equation}
Since $S$ is the covariance matrix of the centered B-spline basis vector, it is positive semidefinite.

We first clarify the degenerate directions. B-spline basis functions satisfy the partition-of-unity property. Therefore, there exists a unit vector $\mathbf{v}_0\in\mathbb{R}^m$, corresponding to the constant spline direction, such that
\begin{equation}
\mathbf{v}_0^\top \mathbf{b}(z)
=
\text{constant}
\end{equation}
for all $z$ in the spline domain. Consequently,
\begin{equation}
S\mathbf{v}_0=\mathbf{0}.
\end{equation}
Let
\begin{equation}
\mathcal{V}_{\perp}
=
\left\{
\mathbf{v}\in\mathbb{R}^m:
\mathbf{v}^\top\mathbf{v}_0=0
\right\}
\end{equation}
be the Euclidean orthogonal complement of the constant spline direction. We assume that $S$ is positive definite on $\mathcal{V}_{\perp}$, which is the standard non-degeneracy condition for the spline Gram matrix.

The degenerate null space of $M$ is
\begin{equation}
\mathcal{Z}
=
\left\{
\mathbf{a}\otimes\mathbf{v}_0:
\mathbf{a}^\top\mathbf{1}_p=0
\right\}.
\label{eq:app_null_space}
\end{equation}
Indeed, for any $\mathbf{a}^\top\mathbf{1}_p=0$, we have
\begin{equation}
(J_p\otimes D)(\mathbf{a}\otimes\mathbf{v}_0)
=
(J_p\mathbf{a})\otimes(D\mathbf{v}_0)
=
\mathbf{0},
\end{equation}
and
\begin{equation}
(R\otimes S)(\mathbf{a}\otimes\mathbf{v}_0)
=
(R\mathbf{a})\otimes(S\mathbf{v}_0)
=
\mathbf{0}.
\end{equation}
Thus, $\dim(\mathcal{Z})=p-1$. The first non-degenerate eigenvalue $\lambda_p(M)$ is the smallest eigenvalue of $M$ restricted to
\begin{equation}
\mathcal{W}
=
\mathcal{Z}^{\perp}.
\end{equation}
Every vector $\mathbf{w}\in\mathcal{W}$ can be written as
\begin{equation}
\mathbf{w}
=
\alpha\,\mathbf{1}_p\otimes\mathbf{v}_0
+
\mathbf{z},
\label{eq:app_lambdap_w_decomp}
\end{equation}
where $\alpha\in\mathbb{R}$ and
\begin{equation}
\mathbf{z}
=
[\mathbf{z}_1^\top,\ldots,\mathbf{z}_p^\top]^\top,
\qquad
\mathbf{z}_j\in\mathcal{V}_{\perp}.
\end{equation}
Since $\mathbf{v}_0$ is orthogonal to $\mathcal{V}_{\perp}$, we have
\begin{equation}
\|\mathbf{w}\|_2^2
=
\alpha^2p
+
\sum_{j=1}^{p}\|\mathbf{z}_j\|_2^2.
\label{eq:app_lambdap_norm_decomp}
\end{equation}

\subsection{Lower Bound of $\lambda_p(M)$}

Let
\begin{equation}
r_{\min}
=
\lambda_{\min}(R),
\qquad
\lambda_C
=
\lambda_{\min}(C).
\end{equation}
Since $R$ is positive definite, $r_{\min}>0$. Moreover, because $R$ is a correlation matrix, $\operatorname{Tr}(R)=p$, and hence $r_{\min}\leq 1$.

For any $\mathbf{w}\in\mathcal{W}$ with the decomposition in \eqref{eq:app_lambdap_w_decomp}, the quadratic form of $M$ satisfies
\begin{align}
\mathbf{w}^{\top}M\mathbf{w}
&=
\mathbf{w}^{\top}(J_p\otimes D)\mathbf{w}
+
\mathbf{w}^{\top}(R\otimes S)\mathbf{w} \notag \\
&\geq
\mathbf{w}^{\top}(J_p\otimes D)\mathbf{w}
+
r_{\min}\,
\mathbf{w}^{\top}(I_p\otimes S)\mathbf{w},
\label{eq:app_lambdap_lower_step1}
\end{align}
where we used $R\succeq r_{\min}I_p$ and $S\succeq 0$.

Because $S\mathbf{v}_0=\mathbf{0}$, the second term depends only on the non-constant components:
\begin{equation}
\mathbf{w}^{\top}(I_p\otimes S)\mathbf{w}
=
\sum_{j=1}^{p}\mathbf{z}_j^\top S\mathbf{z}_j .
\label{eq:app_lambdap_S_term}
\end{equation}
For the first term, using $D=\boldsymbol{\nu}\boldsymbol{\nu}^\top$, we have
\begin{align}
\mathbf{w}^{\top}(J_p\otimes D)\mathbf{w}
&=
\left(
\sum_{j=1}^{p}
(\alpha\mathbf{v}_0+\mathbf{z}_j)
\right)^\top
D
\left(
\sum_{j=1}^{p}
(\alpha\mathbf{v}_0+\mathbf{z}_j)
\right) \notag \\
&=
\left[
\boldsymbol{\nu}^\top
\left(
p\alpha\mathbf{v}_0+\sum_{j=1}^{p}\mathbf{z}_j
\right)
\right]^2 .
\label{eq:app_lambdap_D_term}
\end{align}

Define
\begin{equation}
\bar{\mathbf{z}}
=
\frac{1}{p}\sum_{j=1}^{p}\mathbf{z}_j,
\qquad
\boldsymbol{\delta}_j
=
\mathbf{z}_j-\bar{\mathbf{z}}.
\end{equation}
Then
\begin{equation}
\sum_{j=1}^{p}\boldsymbol{\delta}_j=\mathbf{0},
\qquad
\bar{\mathbf{z}}\in\mathcal{V}_{\perp},
\qquad
\boldsymbol{\delta}_j\in\mathcal{V}_{\perp}.
\end{equation}
Furthermore,
\begin{equation}
\sum_{j=1}^{p}\|\mathbf{z}_j\|_2^2
=
p\|\bar{\mathbf{z}}\|_2^2
+
\sum_{j=1}^{p}\|\boldsymbol{\delta}_j\|_2^2,
\label{eq:app_lambdap_z_norm_mean_dev}
\end{equation}
and
\begin{equation}
\sum_{j=1}^{p}\mathbf{z}_j^\top S\mathbf{z}_j
=
p\bar{\mathbf{z}}^\top S\bar{\mathbf{z}}
+
\sum_{j=1}^{p}\boldsymbol{\delta}_j^\top S\boldsymbol{\delta}_j .
\label{eq:app_lambdap_z_S_mean_dev}
\end{equation}

Substituting \eqref{eq:app_lambdap_D_term} and \eqref{eq:app_lambdap_z_S_mean_dev} into \eqref{eq:app_lambdap_lower_step1}, we obtain
\begin{align}
\mathbf{w}^{\top}M\mathbf{w}
&\geq
p^2
\left[
\boldsymbol{\nu}^\top
(\alpha\mathbf{v}_0+\bar{\mathbf{z}})
\right]^2
+
r_{\min}p\bar{\mathbf{z}}^\top S\bar{\mathbf{z}}
+
r_{\min}
\sum_{j=1}^{p}
\boldsymbol{\delta}_j^\top S\boldsymbol{\delta}_j .
\label{eq:app_lambdap_lower_step2}
\end{align}

Since $r_{\min}\leq 1\leq p$, we have
\begin{align}
&p^2
\left[
\boldsymbol{\nu}^\top
(\alpha\mathbf{v}_0+\bar{\mathbf{z}})
\right]^2
+
r_{\min}p\bar{\mathbf{z}}^\top S\bar{\mathbf{z}}
\notag \\
&\quad
\geq
r_{\min}p
\left\{
\bar{\mathbf{z}}^\top S\bar{\mathbf{z}}
+
\left[
\boldsymbol{\nu}^\top
(\alpha\mathbf{v}_0+\bar{\mathbf{z}})
\right]^2
\right\}.
\label{eq:app_lambdap_mean_lower}
\end{align}
Because $S\mathbf{v}_0=\mathbf{0}$ and $D=\boldsymbol{\nu}\boldsymbol{\nu}^\top$, the expression in braces is exactly the quadratic form of $C=S+D$ along the vector $\alpha\mathbf{v}_0+\bar{\mathbf{z}}$:
\begin{align}
\bar{\mathbf{z}}^\top S\bar{\mathbf{z}}
+
\left[
\boldsymbol{\nu}^\top
(\alpha\mathbf{v}_0+\bar{\mathbf{z}})
\right]^2
&=
(\alpha\mathbf{v}_0+\bar{\mathbf{z}})^\top
C
(\alpha\mathbf{v}_0+\bar{\mathbf{z}}) \notag \\
&\geq
\lambda_C
\left(
\alpha^2+\|\bar{\mathbf{z}}\|_2^2
\right),
\label{eq:app_lambdap_C_lower_mean}
\end{align}
where we used $\mathbf{v}_0^\top\bar{\mathbf{z}}=0$.

It remains to control the deviation components. Since $S$ is positive definite on $\mathcal{V}_{\perp}$, define
\begin{equation}
s_{\perp}
=
\min_{\mathbf{v}\in\mathcal{V}_{\perp},\,\|\mathbf{v}\|_2=1}
\mathbf{v}^\top S\mathbf{v}.
\end{equation}
The matrix $C=S+D$ is a rank-one positive semidefinite perturbation of $S$. Since $S$ has one zero eigenvalue associated with the constant direction $\mathbf{v}_0$, the Cauchy interlacing theorem implies
\begin{equation}
\lambda_C
=
\lambda_{\min}(C)
\leq
s_{\perp}.
\label{eq:app_lambdap_interlacing}
\end{equation}
Therefore, for every $\boldsymbol{\delta}_j\in\mathcal{V}_{\perp}$,
\begin{equation}
\boldsymbol{\delta}_j^\top S\boldsymbol{\delta}_j
\geq
\lambda_C
\|\boldsymbol{\delta}_j\|_2^2.
\label{eq:app_lambdap_dev_lower}
\end{equation}

Combining \eqref{eq:app_lambdap_lower_step2}, \eqref{eq:app_lambdap_mean_lower}, \eqref{eq:app_lambdap_C_lower_mean}, and \eqref{eq:app_lambdap_dev_lower}, we obtain
\begin{align}
\mathbf{w}^{\top}M\mathbf{w}
&\geq
r_{\min}p\lambda_C
\left(
\alpha^2+\|\bar{\mathbf{z}}\|_2^2
\right)
+
r_{\min}\lambda_C
\sum_{j=1}^{p}
\|\boldsymbol{\delta}_j\|_2^2 \notag \\
&=
r_{\min}\lambda_C
\left(
\alpha^2p
+
p\|\bar{\mathbf{z}}\|_2^2
+
\sum_{j=1}^{p}
\|\boldsymbol{\delta}_j\|_2^2
\right) \notag \\
&=
r_{\min}\lambda_C
\|\mathbf{w}\|_2^2,
\label{eq:app_lambdap_rayleigh_lower}
\end{align}
where the last equality follows from \eqref{eq:app_lambdap_norm_decomp} and \eqref{eq:app_lambdap_z_norm_mean_dev}.

Thus, for every nonzero $\mathbf{w}\in\mathcal{W}$,
\begin{equation}
\frac{\mathbf{w}^{\top}M\mathbf{w}}
{\mathbf{w}^{\top}\mathbf{w}}
\geq
\lambda_{\min}(R)\lambda_{\min}(C).
\end{equation}
Since $\lambda_p(M)$ is the smallest eigenvalue of $M$ on the non-degenerate subspace $\mathcal{W}$, the Rayleigh--Ritz theorem gives
\begin{equation}
\lambda_p(M)
\geq
\lambda_{\min}(R)\lambda_{\min}(C).
\label{eq:app_lambdap_lower_final}
\end{equation}

\subsection{Upper Bound of $\lambda_p(M)$}

We next derive an upper bound using a test direction in the non-degenerate subspace. Let $\mathbf{u}_{\min}\in\mathbb{R}^{p}$ be a unit eigenvector of $R$ corresponding to $\lambda_{\min}(R)$:
\begin{equation}
R\mathbf{u}_{\min}
=
\lambda_{\min}(R)\mathbf{u}_{\min},
\qquad
\|\mathbf{u}_{\min}\|_2=1.
\end{equation}

We choose a unit vector $\mathbf{v}\in\mathcal{V}_{\perp}$ satisfying
\begin{equation}
\boldsymbol{\nu}^\top\mathbf{v}=0.
\label{eq:app_lambdap_v_condition}
\end{equation}
Such a vector exists whenever the spline basis dimension is at least three, i.e., $m\geq 3$, or more generally whenever
\begin{equation}
\mathcal{V}_{\perp}\cap \boldsymbol{\nu}^{\perp}
\end{equation}
is nontrivial. This is the usual setting for practical KANs, where the number of spline basis functions is larger than two.

Consider the test vector
\begin{equation}
\mathbf{w}
=
\mathbf{u}_{\min}\otimes\mathbf{v}.
\end{equation}
Since $\mathbf{v}\in\mathcal{V}_{\perp}$, the vector $\mathbf{w}$ is orthogonal to the degenerate null space $\mathcal{Z}$ and hence belongs to $\mathcal{W}=\mathcal{Z}^{\perp}$. Moreover,
\begin{equation}
\|\mathbf{w}\|_2
=
\|\mathbf{u}_{\min}\|_2
\|\mathbf{v}\|_2
=
1.
\end{equation}

The Rayleigh quotient of $M$ along this direction is
\begin{align}
\mathbf{w}^{\top}M\mathbf{w}
&=
(\mathbf{u}_{\min}^{\top}J_p\mathbf{u}_{\min})
(\mathbf{v}^{\top}D\mathbf{v})
+
(\mathbf{u}_{\min}^{\top}R\mathbf{u}_{\min})
(\mathbf{v}^{\top}S\mathbf{v}) \notag \\
&=
(\mathbf{u}_{\min}^{\top}J_p\mathbf{u}_{\min})
(\boldsymbol{\nu}^{\top}\mathbf{v})^2
+
\lambda_{\min}(R)
\mathbf{v}^{\top}S\mathbf{v}.
\label{eq:app_lambdap_upper_rayleigh}
\end{align}
By the choice of $\mathbf{v}$ in \eqref{eq:app_lambdap_v_condition}, the first term vanishes. Hence
\begin{equation}
\mathbf{w}^{\top}M\mathbf{w}
=
\lambda_{\min}(R)
\mathbf{v}^{\top}S\mathbf{v}.
\end{equation}
Furthermore, since $\boldsymbol{\nu}^{\top}\mathbf{v}=0$, we have
\begin{equation}
\mathbf{v}^{\top}D\mathbf{v}
=
(\boldsymbol{\nu}^{\top}\mathbf{v})^2
=
0,
\end{equation}
and therefore
\begin{equation}
\mathbf{v}^{\top}S\mathbf{v}
=
\mathbf{v}^{\top}(C-D)\mathbf{v}
=
\mathbf{v}^{\top}C\mathbf{v}.
\end{equation}
It follows that
\begin{equation}
\mathbf{w}^{\top}M\mathbf{w}
=
\lambda_{\min}(R)
\mathbf{v}^{\top}C\mathbf{v}
\leq
\lambda_{\min}(R)\lambda_{\max}(C).
\label{eq:app_lambdap_test_upper}
\end{equation}

Since $\lambda_p(M)$ is the minimum Rayleigh quotient of $M$ over the non-degenerate subspace $\mathcal{W}$, we have
\begin{align}
\lambda_p(M)
&=
\min_{\mathbf{w}\in\mathcal{W},\,\|\mathbf{w}\|_2=1}
\mathbf{w}^{\top}M\mathbf{w} \notag \\
&\leq
(\mathbf{u}_{\min}\otimes\mathbf{v})^\top
M
(\mathbf{u}_{\min}\otimes\mathbf{v}) \notag \\
&\leq
\lambda_{\min}(R)\lambda_{\max}(C).
\label{eq:app_lambdap_upper_final}
\end{align}
Therefore,
\begin{equation}
\lambda_p(M)
\leq
\lambda_{\min}(R)\lambda_{\max}(C).
\end{equation}

Combining the lower bound in \eqref{eq:app_lambdap_lower_final} and the upper bound in \eqref{eq:app_lambdap_upper_final}, we obtain
\begin{equation}
\lambda_{\min}(R)\lambda_{\min}(C)
\leq
\lambda_p(M)
\leq
\lambda_{\min}(R)\lambda_{\max}(C).
\label{eq:app_lambdap_final}
\end{equation}
This completes the derivation of the bounds for the first non-degenerate eigenvalue of the Hessian matrix.

\bibliographystyle{IEEEtran}

\bibliography{reference.bib}

\end{document}